\documentclass{article}

\usepackage{microtype}
\usepackage{graphicx}
\usepackage{rotating} % for sideways tables
\usepackage{placeins} % for floatbarrier
\usepackage{subcaption}
\usepackage{booktabs} % for professional tables

\usepackage[shortlabels]{enumitem}
\usepackage{multirow} % for table cells spanning multiple rows
\usepackage{makecell}  % For word-wrapping in tables

\usepackage{hyperref}

\usepackage[preprint]{icml2026}

\usepackage{amsmath}
\usepackage{amssymb}
\usepackage{mathtools}
\usepackage{amsthm}
\usepackage{thmtools}

\usepackage[capitalize,noabbrev]{cleveref}

\renewcommand{\algorithmicrequire}{\textbf{Input:}}
\renewcommand{\algorithmicensure}{\textbf{Output:}}

\DeclareRobustCommand{\PRE}[3]{#2} % set up for citation

\theoremstyle{plain}
\newtheorem{theorem}{Theorem}[section]

\newtheorem{lemma}[theorem]{Lemma}

\theoremstyle{definition}

\newtheorem{assumption}[theorem]{Assumption}
\theoremstyle{remark}
\newtheorem{remark}[theorem]{Remark}

\newtheorem{innercustomlemma}{Lemma}
\newenvironment{manuallemma}[1]{%
  \renewcommand\theinnercustomlemma{#1}%
  \renewcommand\theHinnercustomlemma{manual.#1}%
  \innercustomlemma
}{\endinnercustomlemma}

\crefname{assumption}{Assumption}{Assumptions}
\Crefname{assumption}{Assumption}{Assumptions}

\usepackage{amsmath,amsfonts,bm}

\def\rvu{{\mathbf{u}}}

\iftutex
    \def\va{{\mathbfit{a}}}
    \def\vb{{\mathbfit{b}}}

    \def\vg{{\mathbfit{g}}}

    \def\vu{{\mathbfit{u}}}

    \def\vx{{\mathbfit{x}}}
    \def\vy{{\mathbfit{y}}}

    \def\mA{{\mathbfit{A}}}

    \def\mI{{\mathbfit{I}}}

    \def\mP{{\mathbfit{P}}}

\else
    \def\va{{\bm{a}}}
    \def\vb{{\bm{b}}}

    \def\vg{{\bm{g}}}

    \def\vu{{\bm{u}}}

    \def\vx{{\bm{x}}}
    \def\vy{{\bm{y}}}

    \def\mA{{\bm{A}}}

    \def\mI{{\bm{I}}}

    \def\mP{{\bm{P}}}

\fi

\def\sR{{\mathcal{R}}}
\def\sS{{\mathcal{S}}}

\def\sU{{\mathcal{U}}}

\def\sX{{\mathcal{X}}}
\def\sY{{\mathcal{Y}}}

\newcommand{\E}{\mathbb{E}}
\newcommand{\R}{\mathbb{R}}

\newcommand{\Var}{\mathrm{Var}}

\DeclareMathOperator*{\argmax}{arg\,max}
\DeclareMathOperator*{\argmin}{arg\,min}

\def\sYfeas{\sY_\mathrm{feas}}
\def\sUfeas#1{\sU_\mathrm{feas}^{#1}}
\def\sXd{\sX_\mathrm{disc}}

\icmltitlerunning{Maximizing Reliability with Bayesian Optimization}

\begin{document}

\twocolumn[
  \icmltitle{Maximizing Reliability with Bayesian Optimization}



  \icmlsetsymbol{equal}{*}

  \begin{icmlauthorlist}
    \icmlauthor{Jack M. Buckingham}{mathsys}
    \icmlauthor{Ivo Couckuyt}{ugent,imec}
    \icmlauthor{Juergen Branke}{wbs}
  \end{icmlauthorlist}

  \icmlaffiliation{mathsys}{Mathsys CDT, University of Warwick, Coventry, UK}
  \icmlaffiliation{wbs}{Warwick Business School, University of Warwick, Coventry, UK}
  \icmlaffiliation{ugent}{Faculty of Engineering and Architecture, Ghent University, Ghent, Belgium}
  \icmlaffiliation{imec}{imec, Leuven, Belgium}

  \icmlcorrespondingauthor{Jack M. Buckingham}{jack.buckingham@warwick.ac.uk}
  \icmlcorrespondingauthor{Juergen Branke}{juergen.branke@wbs.ac.uk}

  \icmlkeywords{Bayesian optimization, knowledge gradient, Thompson sampling, reliability, feasibility, yield optimization, rare events}

  \vskip 0.3in
]



\printAffiliationsAndNotice{}  % no special notice (required even if empty)

\begin{abstract}
  Bayesian optimization~(BO) is a popular, sample-efficient technique for expensive, black-box optimization.
  One such problem arising in manufacturing is that of maximizing the reliability, or equivalently minimizing the probability of a failure, of a design which is subject to random perturbations -- a problem that can involve extremely rare failures ($P_\mathrm{fail} = 10^{-6}-10^{-8}$).
  In this work, we propose two BO methods based on Thompson sampling and knowledge gradient, the latter approximating the one-step Bayes-optimal policy for minimizing the logarithm of the failure probability.
  Both methods incorporate importance sampling to target extremely small failure probabilities.
  Empirical results show the proposed methods outperform existing methods in both extreme and non-extreme regimes.
\end{abstract}

\section{Introduction} \label{sec:introduction}
The reliability of a system is quantified by the probability that a design satisfies all performance constraints when that performance depends on random inputs or disturbances.
In many applications it is natural to seek to maximize reliability, or equivalently minimize the probability of failure.
Examples include yield optimization of electronic components \citep{wang2018yieldopt,weller2022fast}, the design of mechanical components \citep{huang2010egoReliability,bichon2012efficient}, the optimization of advertising campaigns \citep{betlei2024rladvertising} and the choice of treatment dosage in a medical setting \citep{durham1998sequential}.

The linear version of this problem was first studied as the `P Model' variant of chance constrained programming by \citet{charnes1963chanceconstraints}.
However, the examples given above are generally non-linear and non-convex, with data acquisition relying on resource-intensive simulation or real-world experiments.

In some cases, the required reliability is very high, meaning importance sampling or other techniques for estimating rare event probabilities must be applied. For example, in electronics, because millions of SRAM cells are stacked together, the failure rate of an individual cell must be at most \(10^{-8}\) -- \(10^{-6}\) to ensure a reasonable failure rate for the whole \citep{sun2015sss}.
Additionally, in both simulations and experiments, gradients are typically unavailable, and we are limited to making a small number of carefully chosen evaluations.
Hence, Bayesian optimization~(BO) is an ideal algorithm \citep{frazier2018botutorial,garnett2023bayesopt}.

Minimizing the failure probability hinges on learning an accurate representation of the relevant parts of the boundary which separates the feasible and infeasible regions. In reliability engineering this is known as the \emph{limit state surface}.
Consider the design of a mechanical or electronic component, where it is not possible to exactly manufacture the nominal design and instead the final design is a perturbation of this nominal design. An efficient optimization algorithm would run experiments on designs which are close to both the limit state surface and the nominal design.
However, existing methods fail to capitalize on this and instead explore the whole limit state surface evenly.
This can be seen in \cref{fig:example-problem}, where knowledge gradient (one algorithm proposed in this paper) is compared with an alternative proposed in \citep{huang2010egoReliability}.

\begin{figure*}[t]
    \centering
    \includegraphics[width=\linewidth,height=.3\linewidth]{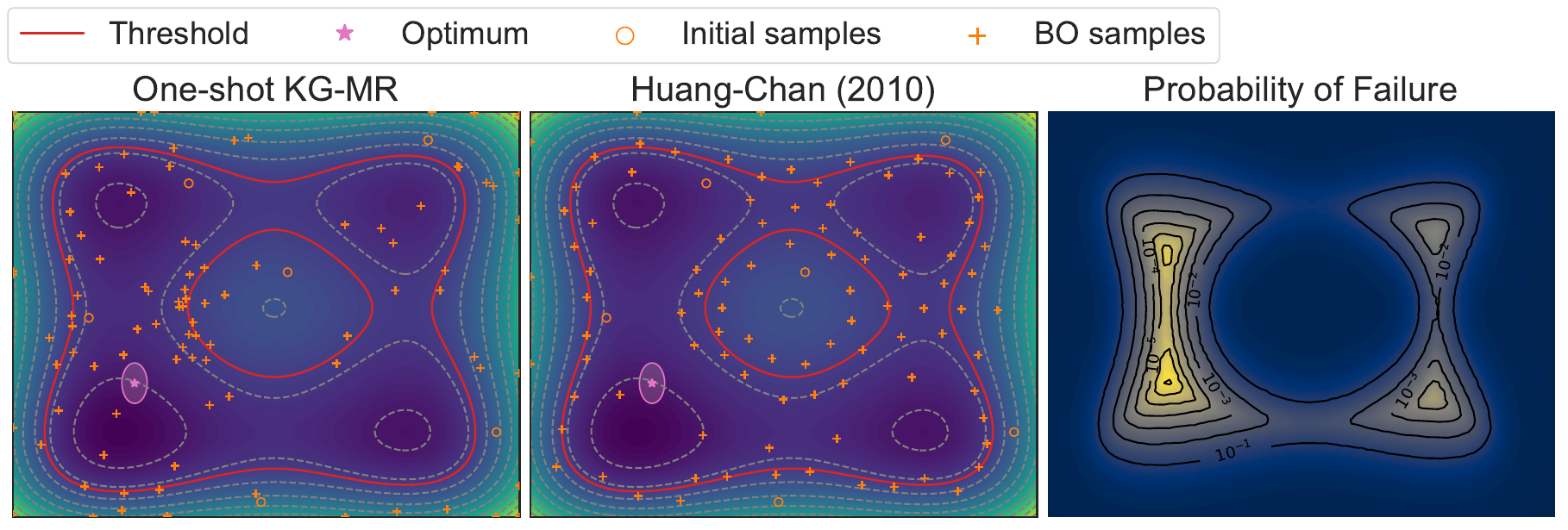}
    \caption[Contour plots showing the Styblinski-Tang (2D) reliability maximization problem.]{Contour plots showing the Styblinski-Tang (2D) problem where the nominal design is perturbed by adding a normally distributed random variable, and perturbed designs with a value above the threshold are classed as failures. The left and middle panels show the black-box function \(f\) and the threshold level set, along with the observations collected by one-shot knowledge gradient for maximal reliability (KG-MR) and the algorithm of \citet{huang2010egoReliability}. The optimal nominal design is shown by the pink dot, surrounded by an ellipse indicating one standard deviation of the normally distributed perturbation. The right panel shows the probability of failure over the domain, as defined in~\eqref{eq:fail-prob-additive}.}
    \label{fig:example-problem}
\end{figure*}

\paragraph{Contributions}
\begin{enumerate}
    \item We propose two acquisition strategies to extend BO to the problem of maximizing reliability, one based on Thompson sampling (called TS-MR) and the other on knowledge gradient (called KG-MR).
    \item We provide one approximation for TS-MR, and two approximations for KG-MR, all of which incorporate importance sampling to extend applicability to the regime of extremely small optimal failure probabilities.
    \item We demonstrate that both TS-MR and KG-MR outperform the state-of-the-art on many examples, with one-shot KG-MR being the most reliable top performer.
\end{enumerate}

\section{Problem Definition} \label{sec:problem-statement}
Let \(\sX \subset \R^{d_x}\) and \(\sY \subset \R^{d_y}\) and \(\sU \subset \R^{d_u}\) denote the spaces of the nominal design, perturbed design and perturbation variables.
Let \(f : \sY \to \R\) be a black-box function and suppose that we can make expensive observations \(v = f(\vy)\) at points \(\vy \in \sY\) of our choice.
Let \(\rvu\) be a random vector taking values in \(\sU\), and let \(\vg: \sX \times \sU \to \sY\) be a known function mapping pairs of nominal designs \(\vx\) and perturbations \(\vu\) to perturbed designs \(\vy\).
Let \(c \in \R\) be a threshold such that perturbed designs \(\vy \in \sY\) with \(f(\vy) \geq c\) are regarded as infeasible.
Finally, let \(\sYfeas \subset \sY\) be a subspace such that values \(\vy \in \sY \setminus \sYfeas\) are regarded as infeasible regardless of the value of \(f(\vy)\).
This includes simpler constraints on \(\vy\) which are not a black-box, such as bound constraints on a normally perturbed nominal design.
The problem tackled in this paper is to minimize the failure probability associated with the nominal design,
\begin{subequations}\label{eq:problem-statement}
\begin{gather}
    \min_{\vx \in \sX} P(\vx) \\
    P(\vx) \coloneqq \mathbb{P}_\rvu\Bigl( f \circ \vg(\vx, \rvu) \geq c \;\text{ or }\; \vg(\vx, \rvu) \notin \sYfeas\Bigr). \label{eq:fail-prob-general}
\end{gather}
\end{subequations}
Here \(f \circ \vg(\vx, \rvu) = f(\vg(\vx, \rvu))\) denotes function composition and \(\mathbb{P}_\rvu\) denotes probability over \(\rvu \sim \mathbb{P}_\rvu\).

The following lemma, which is proved in \cref{sec:theoretical-results}, ensures that Problem~\eqref{eq:problem-statement} has at least one solution.
\begin{lemma}\label{thm:true-fail-prob-cts}
    Suppose that \(f: \sY \to \R\) is continuous, \({\vg : \sX \times \sU \to \sY}\) is continuous and that the distribution of \(\mathbb{P}_\rvu\) has no mass on the limit state surface or boundary of \(\sYfeas\). Then the failure probability \(P : \sX \to [0, 1]\) is continuous. If also \(\sX\) is compact then \(P\) attains its minimum value for some \(\vx^* \in \sX\).
\end{lemma}

In the experiments, we focus on the case where \(d_x = d_y = d_u = d\), \(\sX = \sYfeas \subset \R^d\) is a hyper-rectangle, \(\sY = \sU = \R^d\) and \(\vg(\vx, \vu) = \vx + \vu\), but the core concepts are applicable in the general case. In this specific case, \eqref{eq:fail-prob-general} becomes
\begin{equation}\label{eq:fail-prob-additive}
    P(\vx) = \mathbb{P}_\rvu\Bigl( f(\vx + \rvu) \geq c \;\text{ or }\; \vx + \rvu \notin \sYfeas \Bigr).
\end{equation}

\section{Related Work} \label{sec:related-work}
BO has been previously applied in reliability-based design optimization \citep{dubourg2011reliabilitybased,janusevskis2013rbdo,elamri2023rbdo,pelamatti2023rbdo} in which the authors seek to minimize an objective in expectation subject to constraints which must hold with some high minimum probability. \citet{tsai2023metamodelbased} solve a similar problem where the constraints must hold in expectation.
In contrast, we consider the problem where the constraints are promoted to the place of the objective, and seek a solution which satisfies a constraint with maximal probability.

Quantile and worst-case optimization are also a closely related problems to which BO has been successfully applied \citep{cakmak2020riskmeasures,nguyen2021valueatrisk,picheny2022quantile,urrehman2014worstcase,han2025worstcase}.
However, here the risk level is fixed and the nominal design \(\vx\) is chosen to give the best possible threshold \(c\), while this paper concerns the reverse problem.

Related problems arise in active learning, including the estimation of the failure probability \(\mathbb{P}(f(\rvu) \geq c)\) (called \emph{reliability analysis}), and the estimation of particular level sets of a black-box function \(\{\vu \in \sU : f(\vu) = c\}\), the limit state surface being such a set.
Works in this area include \citet{picheny2010adaptive}, \citet{bichon2008egra}, \citet{bect2011probFailure}, \citet{knudde2019feasibleRegionDiscovery} and \citet{booth2025failProbEstim}.

The first and second order reliability methods (FORM / SORM) \citep{hasofer1974form,breitung1984sorm} capitalize on the uneven contribution of different parts of the limit state surface to the failure probability by replacing it with a local linear or quadratic approximation about the point on the limit state surface with highest probability density \(p(\vu \,|\, f(\vu) = c )\).
While this is an acceptable approximation for many estimation problems, in the optimization case the optimal nominal design \(\vx^*\) will instead be trapped between \(d+1\) equally most-probable points -- the minimum number of vertices for a polyhedron in \(\R^d\).
While an algorithm which attempts to track these points as the proposed nominal design converges is conceivable, it is not the approach we follow in this paper.

As noted in the introduction, we are particularly interested in the case where the failure probability is very small, which demands techniques from rare-event estimation such as importance sampling \citep{kloek1978importancesampling} or subset simulation \citep{au2001subsetSimulation}.
Modern approaches to both these use Markov chain Monte Carlo~(MCMC) \citep{papaioannou2015MCMCforSubsetSim,papakonstantinou2023mcmcRareEvents,tong2023largeDeviationsAdaptiveIS} and density estimation \citep{deboer2005crossentropy,kroese2013crossentropy}.
\citet{tabandeh2022importanceReview} give a review of modern importance sampling methods.
Importance sampling has been applied to the problem of reliability-based design optimization by \citet{fonseca2006efficient} in a manner agnostic to the presence of a surrogate model.

Among works which seek to maximize reliability, \citet{wang2018yieldopt} use BO to decide where to run an adaptive Monte Carlo estimate of the reliability, while \citet{weller2022fast} propose a nested scheme using BO to train a Gaussian process which is accurate on the limit state surface then apply an evolutionary algorithm to select promising nominal designs.
They require 10,000s and 1,000s of simulations respectively and assume that the simulations are sufficiently fast that this is feasible.
To the best of our knowledge, two works exist which don't require so many simulations.
\citet{huang2010egoReliability} propose an algorithm which chooses between four different acquisition criteria to explore the limit state surface.
\citet{bichon2012efficient} propose three extensions to their earlier work on efficient global reliability analysis (EGRA) to tackle various problems related to reliability-based design optimization, one of which is a constrained version of~\eqref{eq:problem-statement}.
These are used as baseline algorithms for the numerical comparison in \cref{sec:experiments}.

\section{Background}
\paragraph{Bayesian Optimization (BO)} is a sample-efficient, global optimization algorithm for black-box functions \citep{frazier2018botutorial,garnett2023bayesopt}.
A prior belief, typically a Gaussian process~(GP), is placed on the objective and an acquisition strategy trades off between exploring regions with high uncertainty and exploiting regions which are already expected to give good values.
In the vanilla case, evaluations are made sequentially, where at each step the hyper-parameters of the probabilistic prior are fitted to the available data, then the acquisition strategy is used to give the next sample location, and finally the black-box function is evaluated at that location.
The whole process is begun with an initial, space-filling design such as a Sobol' sample, and is typically run until an evaluation budget is exhausted.

Thompson sampling~(TS) is one such acquisition strategy, where the next sample location is the maximum of a random sample from the posterior for the objective.
Conversely, knowledge gradient~(KG) \citep{frazier2009kg,scott2011ctskg} is the one-step look-ahead Bayes-optimal policy when the final recommendation is the maximum of the posterior mean of the objective.

\paragraph{Importance Sampling} is a variance reduction technique for Monte Carlo~(MC) estimates of rare event probabilities \citep{tabandeh2022importanceReview}.
For modest sample sizes, \(N_u\), the standard MC estimate \({P(\vx) \approx \frac{1}{N}\sum_{i=1}^{N_u} \mathbb{I}_{\Omega_i}}\) will likely observe no failures, where here we have written \({\Omega_i = \{f \circ \vg(\vx, \vu_i) \geq c \text{ or } \vg(\vx, \vu_i) \notin \sYfeas\}}\).
Indeed, the standard deviation of this estimator is \(\sqrt{P(\vx)(1-P(\vx)) / N_u}\).
If the true failure probability at the optimal \(\vx^*\) is \(P(\vx^*) = 10^{-7}\) then to get the standard deviation of the estimator below \(10\%\) of the failure probability, we need \(N_u \gtrsim 10^9\) samples.

Importance sampling reduces the variance by sampling from a distribution \(\mathbb{Q}_\rvu\) under which failures occur more frequently, then reweighting samples to account for their true probabilities under \(\mathbb{P}_\rvu\).
Assuming that \(\mathbb{P}_\rvu\) and \(\mathbb{Q}_\rvu\) have densities denoted \(p(\cdot)\) and \(q(\cdot)\) respectively, the importance sampling estimate is
\(P(\vx) \approx \frac{1}{N_u} \sum_{i=1}^{N_u} \frac{p(\vu_i)}{q(\vu_i)} \mathbb{I}_{\Omega_i}\).
A simple choice for \(\mathbb{Q}_\rvu\) is to use the same parametric family as \(\mathbb{P}_\rvu\) with increased variance.
For example, if \(\mathbb{P}_\rvu = \mathcal{N}(0, \Sigma_u)\) then we choose \(\mathbb{Q}_\rvu = \mathcal{N}(0, \tau^2 \Sigma_u)\).

\section{Recommended Design}
\label{sec:recommended-design}
We begin by assuming that we have made a number of observations of the black-box function \(f\) and wish to extract a recommended nominal design.
Problem~\eqref{eq:problem-statement} is one of partial or indirect information, where rather than observing the objective \(P(\vx)\) we instead observe \(f(\vy)\) at a finite set of \(\vy\).
Therefore, we have never collected enough information to fully know \(P(\vx)\) and must resort to a model.

Let us model the black-box objective \(f: \sY \to \R\) using a GP prior, \(f \sim \mathcal{GP}(\mu, k)\). To ensure the optimization problems are well defined, we assume that \(f\) is sample-continuous.

Given this prior belief, we use the posterior mean of the GP as a surrogate for \(f\).
Therefore, after \(n\) observations \(v_1 = f(\vy_1), \dots, v_n = f(\vy_n)\), the recommended nominal design is
\begin{equation} \label{eq:recommendation}
    \vx^*_n \in \argmin_{x \in \sX} P_n(\vx) \coloneqq \E[P(\vx) \,|\, \mathcal{F}_n]
\end{equation}
where \(\mathcal{F}_n\) is the sigma-algebra generated by the \(n\) observations and \(P_n(\vx)\) is the expected failure probability.
\Cref{thm:pred-fail-prob-cts} establishes that the arg-min in~\eqref{eq:recommendation} is non-empty.

The following lemma is proved in \cref{sec:theoretical-results}.
\begin{lemma} \label{thm:pred-fail-prob-cts}
    Suppose that \(f\) is a sample-continuous Gaussian process and that \(\vg: \sX \times \sU \to \sY\) is continuous. Then under mild regularity conditions on \(\mathbb{P}_\rvu\) made explicit in \cref{sec:theoretical-results}, \(P_n: \sX \to [0, 1]\) is continuous. If also \(\sX\) is compact then \(P_n\) attains its minimum value for some \(\vx^*_n \in \sX\).
\end{lemma}

Writing~\eqref{eq:fail-prob-general} as an expectation over a combination of indicators and appealing to Fubini's theorem to switch the order of integration, we obtain
\begin{align}
    &P_n(\vx) = \E[P(\vx) \,|\, \mathcal{F}_n] \nonumber \\
    &= \E\Bigl[ \E_\rvu \Bigl[ \mathbb{I}_{\{f \circ \vg(\vx, \rvu) \geq c\}} \mathbb{I}_{\{\vg(\vx, \rvu) \in \sYfeas\}} \nonumber \\
    &\qquad\qquad\qquad\qquad + \mathbb{I}_{\{\vg(\vx, \rvu) \notin \sYfeas\}}\Bigr] \,\Big|\, \mathcal{F}_n \Bigr] \nonumber \\ % \label{eq:fail-prob-via-indicators-v1} \\
    &= \E_\rvu\Bigl[ \Phi_n(\vx, \rvu; c) \, \mathbb{I}_{\{\vg(\vx, \rvu) \in \sYfeas\}}
    + \mathbb{I}_{\{\vg(\vx, \rvu) \notin \sYfeas\}} \Bigr] \label{eq:fail-prob-via-indicators}
\end{align}
where we write
\begin{equation} \label{eq:phi-n}
    \Phi_n(\vx, \vu; c) = \Phi\Biggl(\frac{\mu_n\bigl(\vg(\vx, \vu)\bigr) - c}{\sqrt{k_n\bigl(\vg(\vx, \vu), \vg(\vx, \vu)\bigr)}} \Biggr)
\end{equation}
and where \(\mu_n\) and \(k_n\) are the posterior mean and covariance functions of \(f\) conditional on \(\mathcal{F}_n\).

Finally, in the case of very small failure probabilities, we care more about the order of magnitude of the failure probability than the failure probability itself.
Thus, we define
\begin{equation} \label{eq:log-fail-prob}
    R_n(\vx) = - \log P_n(\vx).
\end{equation}
The negative here simply converts from a minimization problem to a maximization problem.
While \(R_n\) is maximized at the same place where \(P_n\) is minimized, taking the logarithm will be convenient since the default convergence tolerances for gradient-based optimizers will likely cause early termination for \(P_n\) but will remain applicable to \(R_n\).
Of course, \(R_n(\vx)\) is only well defined if \(P_n(\vx) > 0\). However, if \(P_n(\vx) = 0\) then \(P(\vx) = 0\) almost surely and so \(\vx\) is an optimal point with zero probability of failure and the optimization problem is solved. Therefore, we can safely ignore this case.

\section{Acquisition Strategies}
\label{sec:acquisition-strategies}
We now introduce the two acquisition strategies proposed in this paper.

\subsection{Thompson Sampling for Maximizing Reliability (TS-MR)}
\label{sec:thompson-sampling}
The first algorithm we propose is based on Thompson sampling.
At each stage of the optimization, we draw a random sample \(\tilde{f}\) from the GP posterior on \(f\) and use this to compute the nominal design with maximal reliability, \(\vx_{n+1}\). 
\begin{multline} \label{eq:thompson-nominal-acqf}
    \vx_{n+1} \in \\
    \argmin_{\vx \in \sX} \mathbb{P}_\rvu\Bigl( \tilde{f} \circ \vg(\vx, \rvu) \geq c \;\text{ or }\; \vg(\vx, \rvu) \notin \sYfeas \Bigr).
\end{multline}
We then choose a perturbation \(\vu_{n+1}\) which maximizes the product of the probability density of that perturbation and the variance of the failure indicator at that perturbation,
\begin{subequations} \label{eq:thompson-perturbation-acqf}
\begin{align}
    &\mkern-18mu \vu_{n+1} \in \argmax_{\vu \in \sUfeas{n+1}} \alpha^\text{MV}_n(\vu; \vx_{n+1}), \label{eq:thompson-perturbation-acqf-optim} \\
    &\mkern-18mu \alpha^\text{MV}_n(\vu; \vx_{n+1}) \nonumber \\
    &= p(\vu) \Var\bigl[ \mathbb{I}_{\{f \circ \vg(\vx_{n+1}, \vu) \geq c\}} \,\big|\, \mathcal{F}_n \bigr] \nonumber \\
    &= p(\vu) \Phi_n(\vx_{n+1}, \vu; c) \Bigl(1 - \Phi_n(\vx_{n+1}, \vu; c) \Bigr). \label{eq:thompson-perturbation-acqf-defn}
\end{align}
\end{subequations}
For the combination to be feasible, the optimization is constrained to lie in \(\sUfeas{n+1} = \vg_{\vx_{n+1}}^{-1}(\sYfeas) \subset \sU\) where we have written \(\vg_{\vx_{n+1}}(\cdot) = \vg(\vx_{n+1}, \cdot)\).
The idea of \(\alpha^\text{MV}_n\) is to encourage exploration of those parts of the limit state surface which are near to the proposed nominal design \(\vx_{n+1}\).
Finally, the next sample location is given by
\begin{equation} \label{eq:thompson-next-sample}
    \vy_{n+1} = \vg(\vx_{n+1}, \vu_{n+1}).
\end{equation}

\subsection{Knowledge Gradient for Maximizing Reliability (KG-MR)} \label{sec:knowledge-gradient-ideal}
As noted in \cref{sec:recommended-design}, the problem in~\eqref{eq:problem-statement} is one of partial or indirect information. KG has previously had much success being applied to such problems \citep{daulton2023hvkg,buathong2024functionnetworks,buckingham2025twostage,buckingham2025decoupledkg}.

KG applied to the value functions \((R_n)_{n \geq 1}\) defined in~\eqref{eq:log-fail-prob} is the change in expected value of the maximum of the negative log-failure probability before and after the next observation, conditional on that observation being made at the specified \(\vy \in \sYfeas\). That is, after \(n\) observations it is the function \(\alpha^\text{KG-MR}_n: \sYfeas \to \R\) given by
\begin{align}
    &\alpha^\text{KG-MR}_n(\vy) = \nonumber \\
    &\quad \E\biggl[ \max_{\vx \in \sX} R_{n+1}(\vx) \,\bigg|\, \mathcal{F}_n, \vy_{n+1} = \vy \biggr] - \max_{\vx \in \sX} R_n(\vx). \label{eq:kg-ideal}
\end{align}

It approximates the one-step look-ahead Bayes-optimal policy for minimizing the logarithm of the failure probability by replacing \(P(\vx)\) with \(P_n(\vx)\) in the expected utility \(\E[-\log P(\vx) \,|\, \mathcal{F}_n] \approx -\log P_n(\vx) = R_n(\vx)\).
This greatly improves computational efficiency and similar plug-in approximations have previously performed well in multi-objective BO \citep{daulton2023hvkg}.
\begin{remark}
    As noted earlier, if \(\max_{\vx \in \sX} R_{n+1}(\vx) = \infty\) then we have found a point with almost surely zero probability of failure and so need not worry that \(\alpha^\text{KG-MR}_n(\vy)\) is not well defined.
\end{remark}
\begin{remark}
    In the case where the optimal failure probability isn't as extreme, we could very well drop the logarithm and define it directly from \(P_n\).
\end{remark}

\begin{restatable}{lemma}{thmkgnonneg} \label{thm:kg-nonneg}
    The KG-MR acquisition function in~\eqref{eq:kg-ideal} is everywhere non-negative. That is,
    \begin{equation*}
        \forall n \geq 0 \; \forall \vy \in \sYfeas \quad \alpha^\text{KG-MR}_n(\vy) \geq 0.
    \end{equation*}
\end{restatable}
This is proved in \cref{sec:theoretical-results}.

\section{Approximations} \label{sec:approximations}
In order to optimize to find recommendations in~\eqref{eq:recommendation} or query points in~\eqref{eq:thompson-nominal-acqf}, \eqref{eq:thompson-perturbation-acqf} and \eqref{eq:kg-ideal}, we will need to approximate~\eqref{eq:fail-prob-via-indicators}, \eqref{eq:log-fail-prob} and the acquisition functions.
A naive approximation for~\eqref{eq:fail-prob-via-indicators} uses an MC approximation over \(N_u\) independent points \(\vu_1, \dots, \vu_{N_u} \sim \mathbb{P}_\rvu\),
\begin{multline} \label{eq:pn-naive-approx}
    P_n(\vx) \approx \frac{1}{N_u} \sum_{i = 1}^{N_u} \Phi_n(\vx, \vu_i; c) \; \mathbb{I}_{\{\vg(\vx, \vu_i) \in \sYfeas\}} \\
    + \mathbb{I}_{\{\vg(\vx, \vu_i) \notin \sYfeas\}}.
\end{multline}
However, there are two problems with optimizing this approximation directly. Firstly, while we were able to get rid of the indicator \(\mathbb{I}_{\{f \circ \vg(\vx, \rvu) \geq c\}}\) analytically thanks to the uncertainty in the GP, we could not remove all the indicators and the remaining ones mean the approximation in~\eqref{eq:pn-naive-approx} is not continuous in \(\vx\). Secondly, if the optimal failure probability is rare, then it is likely that all of the \(\vu_i\) will lie well below the failure threshold, meaning \(\Phi_n(\vx, \vu_i; c) \approx 0\) has almost zero gradient and we will have no gradient information to use to optimize \(P_n(\vx)\).

To address the first of these issues, we use an ad-hoc smoothing of the indicator function, described in \cref{sec:further-experimental-details--indicator-smoothing}.
Specifically, we approximate \(\mathbb{I}_{\{\vg(\vx, \vu) \in \sYfeas\}} \approx \iota(\vg(\vx, \vu); \delta)\), where \(\delta\) determines the width of the smoothed region around the boundary \(\partial \sYfeas\).
In practice, we find a very small amount of smoothing to be sufficient and set \(\delta=\min(0.05 \ell_\text{min}, 0.1)\) where \(\ell_\text{min}\) is the smallest side-length of \(\sYfeas\).

For the second issue, caused by near zero contributions to the failure probability, we use importance sampling with sampling distribution \(\mathbb{Q}_\rvu = \mathcal{N}(0, \tau^2 \Sigma_u)\).
In practice, we find a value of \(\tau = 3\) works well with optimal failure probabilities in the range \(10^{-6} - 10^{-8}\).

\subsection{Approximation for Recommendations} \label{sec:approximations-recommendation}
Introducing the indicator smoothing approximations and importance sampling into~\eqref{eq:recommendation} gives the following estimates for \(P_n(\vx)\), \(R_n(\vx)\) and the recommendation \(\vx^*_n\),
\begin{subequations}\label{eq:recommendations-approx}
\begin{align}
    \hat{\vx}^*_n &\in \argmax_{\vx \in \sX} \hat{R}_n(\vx) = \argmin_{\vx \in \sX} \hat{P}_n(\vx), \label{eq:recommendations-approx-argmin}\\
    \hat{R}_n(\vx) &= -\log{\hat{P}_n(\vx)}, \label{eq:recommendations-approx-rhat}\\
    \hat{P}_n(\vx) &= \frac{1}{N_u} \sum_{i=1}^{N_u} \frac{p(\vu_i)}{q(\vu_i)} \hat{J}_n(\vx, \vu_i), \label{eq:recommendations-approx-mcapprox} \\
    \hat{J}_n(\vx, \vu) &= \Phi_n(\vx, \vu; c) \iota\bigl(\vg(\vx, \vu); \delta \bigr) \nonumber \\
    &\qquad\qquad\qquad + \bigl(1 - \iota\bigl(\vg(\vx, \vu); \delta\bigr) \bigr). \label{eq:recommendations-approx-jhat}
\end{align}
\end{subequations}
We solve~\eqref{eq:recommendations-approx-argmin} via multi-start L-BFGS-B, a deterministic gradient-based optimizer (see \cref{sec:further-experimental-details--optimizing-recommendations}).

Further, it is known that quasi-Monte Carlo~(qMC) improves the convergence rate of MC estimates from \(\mathcal{O}(1/\sqrt{N_u})\) to \(\mathcal{O}((\log N_u)^{d_u}/N_u)\) for many integrands \citep{lemieux2009montecarlo}.
Where the distribution \(\mathbb{P}_\rvu\) is amenable, we use a qMC approximation for~\eqref{eq:recommendations-approx-mcapprox} using a Sobol' sample in place of independent samples. This is possible for Gaussian and uniformly distributed \(\rvu\). 

\subsection{Approximation for Thompson Sampling}
The first approximation required in TS-MR is to draw the sample \(\tilde{f}\) conditional on \(\mathcal{F}_n\). For this, we use the pathwise sampling method of \citet{wilson2020gpsampling,wilson2021pathwise} with \(N_\mathrm{rff} = 1024\) random Fourier features~(RFFs).
We use the following approximation, analogous to~\eqref{eq:recommendations-approx}, using \(\tilde{P}(\vx)\) to denote the failure probability in~\eqref{eq:thompson-nominal-acqf} obtained using \(\tilde{f}\),
\begin{subequations}\label{eq:thompson-nominal-acqf-approximation}
{\allowdisplaybreaks
\begin{align}
    \tilde{P}(\vx) &\approx \frac{1}{N_u}\sum_{i=1}^{N_u} \frac{p(\vu_i)}{q(\vu_i)} \tilde{J}(\vx, \vu_i) \\
    \tilde{J}(\vx, \vu) &\approx \Phi\Biggl(\frac{\tilde{f} \circ \vg(\vx, \vu) - c}{\rho}\Biggr) \iota\bigl(\vg(\vx, \vu); \delta\bigr) \nonumber \\
    &\qquad\qquad\qquad\quad + \bigl(1 - \iota\bigl(\vg(\vx, \vu); \delta\bigr)\bigr).
\end{align}
}
\end{subequations}

Here we have mimicked the smoothing of \(\mathbb{I}_{\{\tilde{f} \circ \vg(\vx, \vu) \geq c\}}\) in \eqref{eq:recommendations-approx-jhat} using a parameter \(\rho > 0\).
Again, we optimize this approximation of \(\tilde{P}(\vx)\) using multi-start L-BFGS-B, performed in log-space.
The second stage of TS-MR requires optimizing \(\alpha^\text{MV}_n\) in~\eqref{eq:thompson-perturbation-acqf}.
This requires neither importance sampling nor indicator smoothing.
However, we do perform it in log-space, using the \texttt{log\_ndtr()} function in PyTorch to compute \(\log \Phi_n(\vx_{n+1}, \vu; c)\) and \(\log{(1 - \Phi_n(\vx_{n+1}, \vu; c))}\) in a numerically stable manner.
The full algorithm is presented in \cref{alg:thompson-sampling}.

\begin{algorithm}[t]
\caption{Thompson Sampling for Maximal Reliability}
\label{alg:thompson-sampling}
\begin{algorithmic}[1]
    \REQUIRE Initial sample size \(n_0\), evaluation budget \(n_\mathrm{tot}\), qMC sample size \(N_u\), number of Fourier features \(N_\mathrm{rff}\), smoothing parameters \(\rho, \delta\)
    % \ENSURE Recommended nominal design \(\hat{\vx}^*_{n_\mathrm{tot}}\)
    \STATE Evaluate \(f\) at \(n_0\) points, chosen according to a scrambled Sobol' sequence on \(\sYfeas\)
    \FOR {\(n = n_0, \dots, n_\mathrm{tot} - 1\)}
        \STATE Fit MAP hyperparameters of GP prior on \(f\)
        \STATE Sample RFF features and weights for \(\tilde{f} \sim \mathcal{GP}(\mu_n, k_n)\)
        \STATE Generate qMC sample \((\vu_1, \dots, \vu_{N_u}) \sim \mathbb{Q}_\rvu\)
        \STATE Optimize \(\vx_{n+1} \gets \argmin_{\vx \in \sX} \log \tilde{P}(\vx)\) using the approximation in~\eqref{eq:thompson-nominal-acqf-approximation}
        \STATE Optimize \(\alpha^\text{MV}_n\) defined in~\eqref{eq:thompson-perturbation-acqf} to give \\ \(\vu_{n+1} \gets \argmax_{\vu \in \sUfeas{n+1}} \log \alpha^\text{MV}_n(\vu ; \vx_{n+1})\)
        \STATE Set \(\vy_{n+1} \gets \vg(\vx_{n+1}, \vu_{n+1})\)
        \STATE Evaluate expensive function, \(f(\vy_{n+1})\)
    \ENDFOR
    \STATE Compute recommendation \(\hat{\vx}^*_{n_\mathrm{tot}}\) using~\eqref{eq:recommendations-approx}
\end{algorithmic}
\end{algorithm}

\subsection{Approximations for Knowledge Gradient}
The `one-shot' approximation for KG \citep{balandat2020botorch} uses an MC or qMC estimate over \(N_v\) fantasy observations for the expectation over the next observation \(v_{n+1}\) in~\eqref{eq:kg-ideal}.
Variables from the maximization for each element of the MC sum are pulled out to the front to create a single, high-dimensional optimization problem,
\begin{multline} \label{eq:kg-oneshot}
    \max_{\vy \in \sYfeas} \hat{\alpha}_n(\vy) =
    \max_{\substack{\vy \in \sYfeas \\ \vx_1, \dots, \vx_{N_v} \in \sX}} \frac{1}{N_v} \sum_{i=1}^{N_v} \hat{R}_{n+1}(\vx_i; \vy, z_i)  \\
    - \max_{\vx \in \sX} \hat{R}_n(\vx).
\end{multline}
Here \(\hat{R}_n(\vx)\) is as in~\eqref{eq:recommendations-approx-rhat} and \(\hat{R}_{n+1}(\vx; \vy, z)\) is the same but assuming the \((n+1)\)-th observation of \(f\) is at \(\vy\) with value \(v = f(\vy) = \mu_n(\vy) + z\sqrt{k_n(\vy, \vy)}\). This use of the reparametrization trick separates the randomness and the dependence on \(\vy\). The sample \(z_1, \dots, z_{N_v}\) is a sample of \(N_v\) standard Gaussian variables. If they are independent then this is a standard MC estimate, however in practice we use a Sobol' sample passed through a Box-Muller transform to give a qMC estimate as is standard in BoTorch \citep{balandat2020botorch}.
The dimension of the optimization problem in~\eqref{eq:kg-oneshot} is \(d_y + N_v d_x\).

\begin{algorithm}[t]
\caption{Knowledge Gradient for Maximal Reliability}
\label{alg:knowledge-gradient}
\begin{algorithmic}[1]
    \REQUIRE Initial sample size \(n_0\), evaluation budget \(n_\mathrm{tot}\), qMC sample sizes \(N_u, N_v\); for discrete KG-MR the discretization size \(N_x\); for one-shot KG-MR the smoothing parameter \(\delta\)
    % \ENSURE Recommended nominal design \(\hat{\vx}^*_{n_\mathrm{tot}}\)
    \STATE Evaluate \(f\) at \(n_0\) points, chosen according to a scrambled Sobol' sequence on \(\sYfeas\)
    \FOR {\(n = n_0, \dots, n_\mathrm{tot} - 1\)}
        \STATE Fit MAP hyperparameters of GP prior on \(f\)
        \STATE Generate qMC samples \((\vu_1, \dots, \vu_{N_u}) \sim \mathbb{Q}_\rvu\) and \((z_1, \dots, z_{N_v}) \sim \mathcal{N}(0, 1)\)
        \STATE If using discrete KG-MR, sample \(\sXd\) of size \(N_x\) using a scrambled Sobol' sequence
        \STATE Optimize \(\vy_{n+1} \gets \argmax_{\vx \in \sYfeas} \hat{\alpha}_n(\vy)\) using either~\eqref{eq:kg-discrete} or~\eqref{eq:kg-oneshot}
        \STATE Evaluate expensive function, \(f(\vy_{n+1})\)
    \ENDFOR
    \STATE Compute recommendation \(\hat{\vx}^*_{n_\mathrm{tot}}\) using~\eqref{eq:recommendations-approx}
\end{algorithmic}
\end{algorithm}

An alternative which avoids the high-dimensional optimization in~\eqref{eq:kg-oneshot} is to replace the maximization over \(\vx \in \sX\) in~\eqref{eq:kg-ideal} by maximization over a finite set \(\sXd \subset \sX\) sampled from a scrambled Sobol' sequence or other space-filling design \citep[e.g.][]{buathong2024functionnetworks,buckingham2025twostage}.
Combining this with the (q)MC estimate for the expectation over \(v_{n+1}\) gives
\begin{equation} \label{eq:kg-discrete}
    \hat{\alpha}_n(\vy) = \frac{1}{N_v} \sum_{i=1}^{N_v} \max_{\vx \in \sXd} \hat{R}_{n+1}(\vx; \vy, z_i) - \max_{\vx \in \sXd} \hat{R}_n(\vx).
\end{equation}
This approximation is generally accurate when \(d_x\) is small, but suffers from the curse of dimensionality for larger \(d_x\).

For the discrete approximation in~\eqref{eq:kg-discrete}, the set \(\sXd\) is fixed so we do not need to worry about differentiability of \(\hat{R}_n\) and \(\hat{R}_{n+1}\) with respect to \(\vx\) and can set the bounds smoothing factor \(\delta=0\).
For the one-shot approximation in~\eqref{eq:kg-oneshot}, we keep \(\delta = \min(0.05 \ell_\text{min}, 0.1)\).

Both approximations are optimized using multi-start L-BFGS-B, paying particular care to the initial points used for the one-shot approximation (see \cref{sec:kg-optimization}).
The full algorithm is summarized in \cref{alg:knowledge-gradient}.

\section{Experiments} \label{sec:experiments}
In order to judge how well the TS-MR and two KG-MR approximations perform, they are compared to two algorithms which are already present in the literature: the algorithm of \citet{huang2010egoReliability} (denoted HC) and efficient global reliability analysis (EGRA) proposed by \citet{bichon2008egra,bichon2012efficient}.
These methods are outlined in \cref{sec:baseline-algos} and parameter choices are detailed in \cref{sec:further-experimental-details--algorithm-parameters}.
Additionally, a comparison with two simpler baseline algorithms is made.
The first is a random sampling approach which selects points \(\vy_1, \vy_2, \dots\) according to a scrambled Sobol' sequence, and the second uses expected improvement~(EI) to search for \(\min_\vy f(\vy)\) rather than minimizing the failure probability.
For all methods compared, the same procedure is used for generating recommendations. That is, we optimize~\eqref{eq:recommendations-approx-argmin} with multi-start L-BFGS-B as described in \cref{sec:further-experimental-details--optimizing-recommendations}.

We compare the algorithms on three random problems sampled from a GP and nine frequently-used synthetic test problems, ranging from two to sixteen input dimensions. The problems all have \(\sX = \sY = \sU\) and \(\vg(\vx, \vu) = \vx + \vu\) to represent a perturbation about a nominal design. The perturbations used are normally distributed, \(\rvu \sim \mathcal{N}(0, \Sigma_u)\), with diagonal covariance matrices \(\Sigma_u\) varying between problems.

Each algorithm is run 30 times on each test problem with a different random seed.
For each repeat, a different initial design is randomly selected using a scrambled Sobol' sample with size depending on the problem dimension and specified in \cref{sec:further-experimental-details--algorithm-parameters}.
For all algorithms, the GP surrogate has a Mat\'ern-\(5/2\) kernel with MAP-estimated hyperparameters fitted as described in \cref{sec:further-experimental-details--model-params}.

\subsection{Evaluation of the Recommendations}

Given a recommended nominal design \(\vx^*_n\), we evaluate it using a qMC approximation of \(P(\vx^*_n)\) with importance sampling, and using the true test function,
\begin{multline}
    P(\vx^*_n) \approx \frac{1}{N_u} \sum_{i=1}^{N_u} \frac{p(\vu_i)}{q(\vu_i)} \Bigl( \\
    \mathbb{I}_{\{f \circ \vg(\vx^*_n, \vu_i) \geq c\}} \mathbb{I}_{\{\vg(\vx^*_n, \vu_i) \in \sYfeas\}} + \mathbb{I}_{\{\vg(\vx^*_n, \vu_i) \notin \sYfeas\}}\Bigr).
\end{multline}
For this, we use \(N_u = 2^{20} \approx 1,000,000\), which is larger than the values we used for either the acquisition function or for generating the recommendation.
The same inflation method and scale factor for the importance sampling are used for this estimate as were used in the acquisition function and to generate recommendations.
Since in all cases \(\mathbb{Q}_\rvu\) is Gaussian, the qMC approximation uses a \(d\)-dimensional Sobol' sample with coordinates passed pairwise through a Box-Muller transform.

\subsection{Results}
\begin{figure*}[t]
    \centering
    \includegraphics[width=\linewidth]{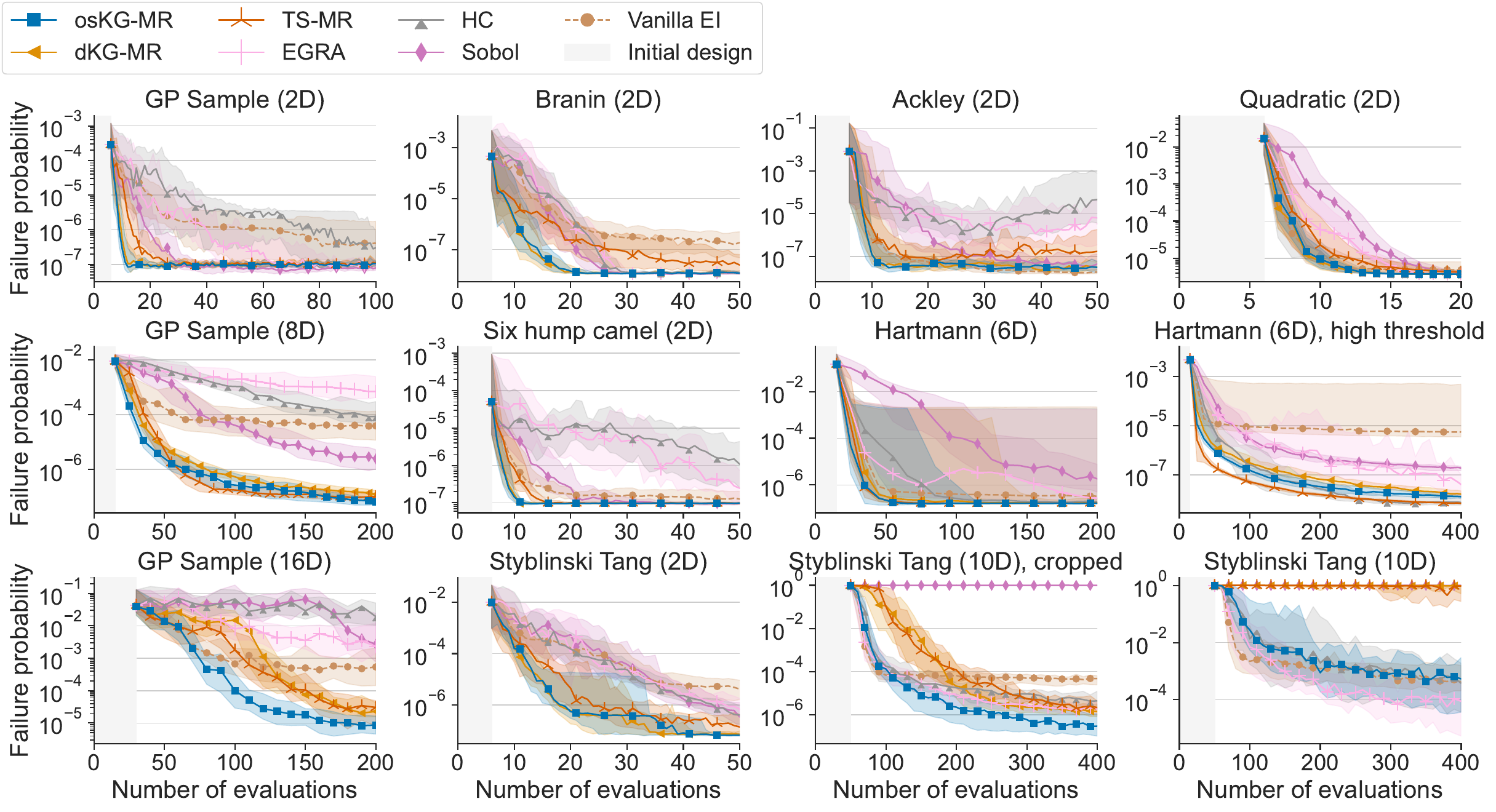}
    \caption[Results on the extreme reliability maximization test problems.]{Failure probabilities for the 12 test problems. The first column contains the GP test problems, while the remaining columns are problems with potential for model mismatch. The last column contains problems for which KG-MR is not expected to have an advantage. The failure probability associated with the recommended solution is shown as a function of number of evaluations of the expensive black-box function. The solid lines show the median failure probability over 30 repeats and the shaded regions show the upper and lower quartiles.}
    \label{fig:results-rare}
\end{figure*}

\Cref{fig:example-problem} shows an example of the samples made by the one-shot KG-MR and HC algorithms on the Styblinski-Tang~(2D) test problem. Both algorithms mainly collect samples near to the limit state surface. However, KG-MR focuses its samples on the parts of the limit state surface which contribute to the failure probability at the optimal design, while HC spreads its samples evenly around the whole limit state surface. This is particularly a problem when the limit state surface intersects the problem boundary (see \cref{sec:example-query-patterns}).

\Cref{fig:results-rare} shows the median failure probability and interquartile range for the recommended design at each point during the BO procedure.
The first column contains the three GP generated problems and the remainder contain the common synthetic problems with the potential for model-mismatch.
The problems in the middle two columns were chosen to have an interesting limit state surface, while the problems in the last column are chosen to exhibit cases where KG-MR and TS-MR are not expected to outperform HC and EGRA.

One-shot KG-MR (osKG-MR) performs best or joint best in 10 of the 12 problems, and discrete KG-MR (dKG-MR) matches its performance in lower dimensions.
In 10+ dimensions, the approximations in dKG-MR are too crude and slow the convergence.
TS-MR also performs strongly, though it lags in the 2D GP, 16D GP, Branin and cropped 10D Styblinski-Tang problems.
HC and EGRA generally perform poorly in the 2D problems because they fail to focus on the relevant parts of the limit state surface (e.g. \cref{fig:example-problem} and \cref{sec:example-query-patterns}).
In higher dimensions, they perform poorly on the GP problems but ok on the synthetic examples.

The problems in the last column of \cref{fig:results-rare} were chosen to exhibit cases where the KG-MR and TS-MR algorithms should not have an advantage.
The 2D quadratic function has a circular limit state surface, meaning all parts of the limit state surface are equally important.
Consequently, the HC algorithm matches the performance of KG-MR and TS-MR, while EGRA is not too far behind.
Similarly, for the 6D Hartmann problem, the higher threshold pushes the limit state surface further towards the edges of the domain and the HC algorithm becomes competitive, although in this case EGRA performs similarly to Sobol' search.
Finally, the HC and EGRA algorithms perform empirically well on the uncropped 10D Styblinski-Tang problem, with EGRA marginally beating HC and osKG-MR. This problem is highly multimodal, with 1024 local minima, of which 848 are feasible.
We hypothesize that HC and EGRA have good local search properties, focusing on learning about the limit state surface around one of the good local minima found in the initial design.
They are also computationally cheaper than osKG-MR (see \cref{sec:acquisition-timings}).
Thompson sampling is known to over-explore in higher dimensions \citep{papenmeier2025exploring}, and makes no progress on this problem, while dKG-MR makes no progress because the approximating set is not able to capture enough of the problem's structure.

Similar conclusions hold in the case of non-extreme failure probabilities, where the importance sampling is not needed and the logarithms in KG-MR are dropped (see \cref{sec:non-extreme-fail-probs}).
Notable differences include osKG-MR being the top or joint-top performer in all problems while TS-MR, HC and EGRA show improvements in some problems. Generally all algorithms converge faster.
A sensitivity study in \cref{sec:sensitivity-study} shows that the KG-MR methods are fairly robust to the choice of the parameters \(N_u\) and \(N_v\), but that the size of the discretization \(N_x = |\sXd|\) matters in higher dimensions for dKG-MR.

\section{Conclusions} \label{sec:conclusion}
This paper has introduced two distinct BO strategies targeting the equivalent problems of maximizing reliability, minimizing failure probability and maximizing yield.
Approximations incorporating quasi-Monte Carlo and importance sampling have been proposed and experiments have been conducted on a range of synthetic problems both with and without model mismatch.
On the 2D test problems, the KG-MR approaches perform equally well and converge faster than any of the alternatives.
On the higher-dimensional problems, the approximations in the discrete KG-MR acquisition function are too coarse and harm the convergence, but the one-shot KG-MR remains the top performer, sometimes jointly with TS-MR.
In all cases, the parameters of the KG-MR are very easy to tune, with the only choice to be made being how to split the computation budget between increasing \(N_u\) and \(N_v\) for osKG-MR in higher dimensions.
The KG-MR algorithms are generally more expensive than the alternatives, in the slowest case taking of the order of 30-40 seconds to optimize using a GPU (see \cref{sec:acquisition-timings}).
In summary, it should be recommended to use one-shot KG-MR unless the computation cost ceases to be negligible compared with the cost of evaluating \(f\), in which case TS-MR is a fast alternative.

The algorithm of \citet{huang2010egoReliability} is designed to tackle problems with multiple limit state functions \(f\), which we haven't addressed empirically in this paper. However, the framework extends very naturally to this case making it a good avenue for future research.
Similarly, constrained versions of the problem such as those tested on in \citet{bichon2012efficient} could be tackled.
Additionally, the framework allows for any combination of the nominal design \(\vx\) and perturbation \(\vu\), and this could be tested experimentally, for example with multiplicative perturbations as has been done in other related works \citep{cakmak2020riskmeasures,daulton2022robustmobo}.
Finally, adaptive importance sampling using kernel density estimation or MCMC could improve efficiency by ensuring all MC points for the estimation of \(\E_\rvu\) are near to the limit state surface.

% \section*{Software and Data}
% \paragraph{Software} Code to reproduce the experiments is available in the supplementary material, and at \url{XXX.github.com/YYY} (Redacted for anonymity during review).

% \paragraph{Data Access Statement}
% This work is entirely theoretical, there is no data underpinning this publication.

\section*{Acknowledgments}
The authors would like to thank Tobias Grafke for an insightful discussion on rare event modeling.
The first author was supported by the Engineering and Physical Sciences Research Council through the Mathematics of Systems II Centre for Doctoral Training at the University of Warwick (reference EP/S022244/1).
The second author was supported by the Flemish Government under the Flanders Artificial Intelligence Research program.
Computing facilities were provided by the Scientific Computing Research Technology Platform of the University of Warwick.

% \textbf{Do not} include acknowledgements in the initial version of the paper

% If a paper is accepted, the final camera-ready version can (and usually should)
% include acknowledgements.  Such acknowledgements should be placed at the end of
% the section, in an unnumbered section that does not count towards the paper
% page limit. Typically, this will include thanks to reviewers who gave useful
% comments, to colleagues who contributed to the ideas, and to funding agencies
% and corporate sponsors that provided financial support.
\section*{Data Access Statement}
This work is entirely theoretical, there is no data underpinning this publication.

\section*{Impact Statement}

This paper presents work whose goal is to advance the field of Machine
Learning. There are many potential societal consequences of our work, none
which we feel must be specifically highlighted here.

\DeclareRobustCommand{\PRE}[3]{#3}
\bibliography{references}

@article{huang2010egoReliability,
  title = {A Modified Efficient Global Optimization Algorithm for Maximal Reliability in a Probabilistic Constrained Space},
  volume = {132},
  issn = {1050-0472},
  url = {https://doi.org/10.1115/1.4001532},
  number = {61002},
  journal = {Journal of Mechanical Design},
  author = {Huang, Yen-Chih and Chan, Kuei-Yuan},
  date = {2010-05-20},
  year = {2010},
}

@article{wang2018yieldopt,
  title = {Efficient Yield Optimization for Analog and {{SRAM}} Circuits via {{Gaussian}} Process Regression and Adaptive Yield Estimation},
  author = {Wang, Mengshuo and Lv, Wenlong and Yang, Fan and Yan, Changhao and Cai, Wei and Zhou, Dian and Zeng, Xuan},
  date = {2018-10},
  year = {2018},
  journal = {IEEE Transactions on Computer-Aided Design of Integrated Circuits and Systems},
  volume = {37},
  number = {10},
  pages = {1929--1942},
  issn = {1937-4151},
  url = {https://doi.org/10.1109/TCAD.2017.2778061},
}

@article{weller2022fast,
  title = {Fast and Efficient High-Sigma Yield Analysis and Optimization Using Kernel Density Estimation on a {{Bayesian}} Optimized Failure Rate Model},
  author = {Weller, Dennis D. and Hefenbrock, Michael and Beigl, Michael and Tahoori, Mehdi B.},
  date = {2022-03},
  year = {2022},
  journal = {IEEE Transactions on Computer-Aided Design of Integrated Circuits and Systems},
  volume = {41},
  number = {3},
  pages = {695--708},
  issn = {1937-4151},
  url = {https://doi.org/10.1109/TCAD.2021.3064440},
}

@article{bichon2012efficient,
  title = {Efficient Global Surrogate Modeling for Reliability-Based Design Optimization},
  author = {Bichon, Barron J. and Eldred, Michael S. and Mahadevan, Sankaran and McFarland, John M.},
  date = {2012},
  year = {2012},
  journal = {Journal of Mechanical Design},
  shortjournal = {J. Mech. Des},
  volume = {135},
  number = {011009},
  issn = {1050-0472},
  url = {https://doi.org/10.1115/1.4022999},
}

@article{bichon2008egra,
  title = {Efficient Global Reliability Analysis for Nonlinear Implicit Performance Functions},
  author = {Bichon, B. J. and Eldred, M. S. and Swiler, L. P. and Mahadevan, S. and McFarland, J. M.},
  date = {2008},
  year = {2008},
  journal = {AIAA Journal},
  volume = {46},
  number = {10},
  pages = {2459--2468},
  publisher = {{American Institute of Aeronautics and Astronautics}},
  issn = {0001-1452},
  url = {https://doi.org/10.2514/1.34321},
}

@article{picheny2010adaptive,
  title = {Adaptive Designs of Experiments for Accurate Approximation of a Target Region},
  author = {Picheny, Victor and Ginsbourger, David and Roustant, Olivier and Haftka, Raphael T. and Kim, Nam-Ho},
  date = {2010-06},
  year = {2010},
  journal = {Journal of Mechanical Design},
  shortjournal = {J. Mech. Des},
  volume = {132},
  number = {071008},
  issn = {1050-0472},
  url = {https://doi.org/10.1115/1.4001873},
}

@article{bect2011probFailure,
  title = {Sequential Design of Computer Experiments for the Estimation of a Probability of Failure},
  author = {Bect, Julien and Ginsbourger, David and Li, Ling and Picheny, Victor and Vazquez, Emmanuel},
  date = {2011-04-21},
  year = {2011},
  journal = {Statistics and Computing},
  shortjournal = {Stat Comput},
  volume = {22},
  number = {3},
  pages = {773--793},
  issn = {1573-1375},
  url = {https://doi.org/10.1007/s11222-011-9241-4},
}

@inproceedings{knudde2019feasibleRegionDiscovery,
  title = {Active Learning for Feasible Region Discovery},
  booktitle = {2019 18th {{IEEE International Conference On Machine Learning And Applications}} ({{ICMLA}})},
  author = {Knudde, Nicolas and Couckuyt, Ivo and Shintani, Kohei and Dhaene, Tom},
  date = {2019-12},
  year = {2019},
  pages = {567--572},
  url = {https://doi.org/10.1109/ICMLA.2019.00106},
  eventtitle = {2019 18th {{IEEE International Conference On Machine Learning And Applications}} ({{ICMLA}})},
}

@article{booth2025failProbEstim,
  title = {Two-Stage Design for Failure Probability Estimation with {{Gaussian}} Process Surrogates},
  author = {Booth, Annie S. and Renganathan, S. Ashwin},
  year = 2025,
  month = oct,
  journal = {Journal of Quality Technology},
  volume = {57},
  number = {5},
  eprint = {https://doi.org/10.1080/00224065.2025.2562868},
  pages = {500--516},
  publisher = {Taylor \& Francis},
  url = {https://doi.org/10.1080/00224065.2025.2562868}
}

@article{hasofer1974form,
  title = {Exact and Invariant Second-Moment Code Format},
  author = {Hasofer, Abraham M. and Lind, Niels C.},
  date = {1974-02-01},
  year = {1974},
  journal = {Journal of the Engineering Mechanics Division},
  volume = {100},
  number = {1},
  pages = {111--121},
  publisher = {American Society of Civil Engineers},
  url = {https://doi.org/10.1061/JMCEA3.0001848},
}

@article{breitung1984sorm,
  title = {Asymptotic Approximations for Multinormal Integrals},
  author = {Breitung, Karl},
  date = {1984-03-01},
  year = {1984},
  journal = {Journal of Engineering Mechanics},
  volume = {110},
  number = {3},
  pages = {357--366},
  publisher = {American Society of Civil Engineers},
  issn = {0733-9399},
  url = {https://doi.org/10.1061/(ASCE)0733-9399(1984)110:3(357)},
}

@article{kloek1978importancesampling,
  title = {Bayesian Estimates of Equation System Parameters: An Application of Integration by {{Monte Carlo}}},
  shorttitle = {Bayesian Estimates of Equation System Parameters},
  author = {Kloek, T. and {\PRE{Dijk}{Van}{van}} Dijk, H. K.},
  date = {1978},
  year = {1978},
  journal = {Econometrica},
  volume = {46},
  number = {1},
  eprint = {1913641},
  eprinttype = {jstor},
  pages = {1--19},
  publisher = {[Wiley, Econometric Society]},
  issn = {0012-9682},
  url = {https://doi.org/10.2307/1913641},
}

@article{au2001subsetSimulation,
  title = {Estimation of Small Failure Probabilities in High Dimensions by Subset Simulation},
  author = {Au, Siu-Kui and Beck, James L.},
  date = {2001-10-01},
  year = {2001},
  journal = {Probabilistic Engineering Mechanics},
  volume = {16},
  number = {4},
  pages = {263--277},
  issn = {0266-8920},
  url = {https://doi.org/10.1016/S0266-8920(01)00019-4},
}

@article{deboer2005crossentropy,
  title = {A Tutorial on the Cross-Entropy Method},
  author = {{\PRE{Boer}{De}{de}} Boer, Pieter-Tjerk and Kroese, Dirk P. and Mannor, Shie and Rubinstein, Reuven Y.},
  date = {2005-02-01},
  year = {2005},
  journal = {Annals of Operations Research},
  shortjournal = {Ann Oper Res},
  volume = {134},
  number = {1},
  pages = {19--67},
  issn = {1572-9338},
  url = {https://doi.org/10.1007/s10479-005-5724-z},
}

@incollection{kroese2013crossentropy,
  title = {Chapter 2 - The Cross-Entropy Method for Estimation},
  booktitle = {Handbook of {{Statistics}}},
  author = {Kroese, Dirk P. and Rubinstein, Reuven Y. and Glynn, Peter W.},
  editor = {Rao, C. R. and Govindaraju, Venu},
  date = {2013-01-01},
  year = {2013},
  series = {Handbook of {{Statistics}}},
  volume = {31},
  pages = {19--34},
  publisher = {Elsevier},
  url = {https://doi.org/10.1016/B978-0-444-53859-8.00002-3},
}

@article{papaioannou2015MCMCforSubsetSim,
  title = {{{MCMC}} Algorithms for Subset Simulation},
  author = {Papaioannou, Iason and Betz, Wolfgang and Zwirglmaier, Kilian and Straub, Daniel},
  date = {2015-07-01},
  year = {2015},
  journal = {Probabilistic Engineering Mechanics},
  volume = {41},
  pages = {89--103},
  issn = {0266-8920},
  url = {https://doi.org/10.1016/j.probengmech.2015.06.006},
}

@article{papakonstantinou2023mcmcRareEvents,
  title = {Hamiltonian {{MCMC}} Methods for Estimating Rare Events Probabilities in High-Dimensional Problems},
  author = {Papakonstantinou, Konstantinos G. and Nikbakht, Hamed and Eshra, Elsayed},
  date = {2023-10-01},
  year = {2023},
  journal = {Probabilistic Engineering Mechanics},
  volume = {74},
  pages = {103485},
  issn = {0266-8920},
  url = {https://doi.org/10.1016/j.probengmech.2023.103485},
}

@article{tong2023largeDeviationsAdaptiveIS,
  title = {Large Deviation Theory-based Adaptive Importance Sampling for Rare Events in High Dimensions},
  author = {Tong, Shanyin and Stadler, Georg},
  date = {2023-09-30},
  year = {2023},
  journal = {SIAM/ASA Journal on Uncertainty Quantification},
  shortjournal = {SIAM/ASA J. Uncertainty Quantification},
  volume = {11},
  number = {3},
  pages = {788--813},
  publisher = {{Society for Industrial and Applied Mathematics}},
  url = {https://doi.org/10.1137/22M1524758},
}

@article{tabandeh2022importanceReview,
  title = {A Review and Assessment of Importance Sampling Methods for Reliability Analysis},
  author = {Tabandeh, Armin and Jia, Gaofeng and Gardoni, Paolo},
  date = {2022-07-01},
  year = {2022},
  journal = {Structural Safety},
  shortjournal = {Structural Safety},
  volume = {97},
  pages = {102216},
  issn = {0167-4730},
  url = {https://doi.org/10.1016/j.strusafe.2022.102216},
}

@article{janusevskis2013rbdo,
  title = {Simultaneous Kriging-Based Estimation and Optimization of Mean Response},
  author = {Janusevskis, Janis and Le Riche, Rodolphe},
  date = {2013-02-01},
  year = {2013},
  journal = {Journal of Global Optimization},
  shortjournal = {J Glob Optim},
  volume = {55},
  number = {2},
  pages = {313--336},
  issn = {1573-2916},
  url = {https://doi.org/10.1007/s10898-011-9836-5},
}

@article{elamri2023rbdo,
  title = {A sampling criterion for constrained Bayesian optimization with uncertainties},
  author = {El Amri, Reda and Le Riche, Rodolphe and Helbert, Céline and Blanchet-Scalliet, Christophette and Da Veiga, Sébastien},
  date = {2023},
  year = {2023},
  journal = {The SMAI Journal of computational mathematics},
  volume = {9},
  pages = {285--309},
  issn = {2426-8399},
  url = {https://doi.org/10.5802/smai-jcm.102},
}

@article{pelamatti2023rbdo,
  title = {Coupling and Selecting Constraints in {{Bayesian}} Optimization under Uncertainties},
  author = {Pelamatti, Julien and Le Riche, Rodolphe and Helbert, Céline and Blanchet-Scalliet, Christophette},
  date = {2023-07-08},
  year = {2023},
  journal = {Optimization and Engineering},
  shortjournal = {Optim Eng},
  volume = {25},
  number = {1},
  pages = {373--412},
  issn = {1573-2924},
  url = {https://doi.org/10.1007/s11081-023-09807-x},
}

@article{fonseca2006efficient,
  title = {Efficient Robust Design via {{Monte Carlo}} Sample Reweighting},
  author = {Fonseca, José R. and Friswell, Michael I. and Lees, Arthur W.},
  date = {2006-08-09},
  year = {2006},
  journal = {International Journal for Numerical Methods in Engineering},
  volume = {69},
  number = {11},
  pages = {2279--2301},
  issn = {1097-0207},
  url = {https://doi.org/10.1002/nme.1850},
}

@article{dubourg2011reliabilitybased,
  title = {Reliability-Based Design Optimization Using Kriging Surrogates and Subset Simulation},
  author = {Dubourg, Vincent and Sudret, Bruno and Bourinet, Jean-Marc},
  date = {2011-11-01},
  year = {2011},
  journal = {Structural and Multidisciplinary Optimization},
  shortjournal = {Struct Multidisc Optim},
  volume = {44},
  number = {5},
  pages = {673--690},
  issn = {1615-1488},
  url = {https://doi.org/10.1007/s00158-011-0653-8},
}

@article{tsai2023metamodelbased,
  title = {Metamodel-Based Simulation Optimization Considering a Single Stochastic Constraint},
  author = {Tsai, Shing Chih and Park, Chuljin and Chang, Min Han},
  date = {2023-07-01},
  year = {2023},
  journal = {Computers \& Operations Research},
  shortjournal = {Computers \& Operations Research},
  volume = {155},
  pages = {106239},
  issn = {0305-0548},
  url = {https://doi.org/10.1016/j.cor.2023.106239},
}

@inproceedings{cakmak2020riskmeasures,
 author = {Cakmak, Sait and Astudillo Marban, Raul and Frazier, Peter I. and Zhou, Enlu},
 booktitle = {Advances in Neural Information Processing Systems},
 pages = {20130--20141},
 publisher = {Curran Associates, Inc.},
 title = {{B}ayesian Optimization of Risk Measures},
 url = {https://proceedings.neurips.cc/paper_files/paper/2020/file/e8f2779682fd11fa2067beffc27a9192-Paper.pdf},
 volume = {33},
 year = {2020}
}

@inproceedings{daulton2023hvkg,
  title = 	 {Hypervolume Knowledge Gradient: A Lookahead Approach for Multi-Objective {B}ayesian Optimization with Partial Information},
  author =       {Daulton, Sam and Balandat, Maximilian and Bakshy, Eytan},
  booktitle = 	 {Proceedings of the 40th International Conference on Machine Learning},
  pages = 	 {7167--7204},
  year = 	 {2023},
  volume = 	 {202},
  series = 	 {Proceedings of Machine Learning Research},
  month = 	 jul,
  publisher =    {PMLR},
  url = 	 {https://proceedings.mlr.press/v202/daulton23a.html},
}

@InProceedings{buathong2024functionnetworks,
  title = {{B}ayesian Optimization of Function Networks with Partial Evaluations},
  author = {Buathong, Poompol and Wan, Jiayue and Astudillo, Raul and Daulton, Sam and Balandat, Maximilian and Frazier, Peter I.},
  booktitle = {International Conference on Machine Learning},
  pages = {4752--4784},
  year = {2024},
  editor = {Salakhutdinov, Ruslan and Kolter, Zico and Heller, Katherine and Weller, Adrian and Oliver, Nuria and Scarlett, Jonathan and Berkenkamp, Felix},
  volume = {235},
  series = {Proceedings of Machine Learning Research},
  month = jul,
  publisher = {PMLR},
  url = {https://proceedings.mlr.press/v235/buathong24a.html},
}

@inproceedings{buckingham2025decoupledkg,
  author={Buckingham, Jack M. and {Rojas Gonzalez}, Sebastian and Branke, Juergen},
  editor={Singh, Hemant and Ray, Tapabrata and Knowles, Joshua and Li, Xiaodong and Branke, Juergen and Wang, Bing and Oyama, Akira},
  title={Knowledge Gradient for Multi-objective {B}ayesian Optimization with Decoupled Evaluations},
  booktitle={Evolutionary Multi-Criterion Optimization},
  series={Lecture Notes in Computer Science},
  volume={15513},
  year={2025},
  month=feb,
  publisher={Springer Nature Singapore},
  address={Singapore},
  pages={117--132},
  isbn={978-981-96-3538-2},
  url={https://doi.org/10.1007/978-981-96-3538-2_9},
}

@article{buckingham2025twostage,
    title={Bayesian Optimization for Non-Convex Two-Stage Stochastic Optimization Problems}, 
    author={Jack M. Buckingham and Ivo Couckuyt and Juergen Branke},
    journal = {arXiv preprint},
    year={2025},
    month=feb,
    url = {https://arxiv.org/abs/2408.17387},
}

@incollection{frazier2018botutorial,
   author = {Peter I. Frazier},
   title = {{B}ayesian Optimization},
   booktitle = {Recent Advances in Optimization and Modeling of Contemporary Problems},
   pages = {255-278},
   publisher = {INFORMS},
   month = oct,
   year = {2018},
   url = {https://doi.org/10.1287/educ.2018.0188},
}

@book{garnett2023bayesopt,
  title={{B}ayesian optimization},
  author={Garnett, Roman},
  year={2023},
  publisher={Cambridge University Press},
  isbn={9781108425780},
  url={https://bayesoptbook.com/}
}

@article{frazier2009kg,
   author = {Frazier, Peter I. and Powell, Warren B. and Dayanik, Savas},
   url = {https://doi.org/10.1287/ijoc.1080.0314},
   issn = {1091-9856},
   issue = {4},
   journal = {INFORMS Journal on Computing},
   month = nov,
   pages = {599-613},
   title = {The Knowledge-Gradient Policy for Correlated Normal Beliefs},
   volume = {21},
   year = {2009},
}

@article{scott2011ctskg,
  author = {Scott, Warren and Frazier, Peter I. and Powell, Warren B.},
  title = {The Correlated Knowledge Gradient for Simulation Optimization of Continuous Parameters using {G}aussian Process Regression},
  journal = {SIAM Journal on Optimization},
  volume = {21},
  number = {3},
  pages = {996-1026},
  year = {2011},
  url = {https://doi.org/10.1137/100801275},
}

@article{charnes1963chanceconstraints,
  title = {Deterministic Equivalents for Optimizing and Satisficing under Chance Constraints},
  author = {Charnes, A. and Cooper, W. W.},
  date = {1963-02},
  year = {1963},
  journal = {Operations Research},
  volume = {11},
  number = {1},
  pages = {18--39},
  publisher = {INFORMS},
  issn = {0030-364X},
  url = {https://doi.org/10.1287/opre.11.1.18},
}

@article{jones1993lipschitzian,
  title = {Lipschitzian Optimization without the {{Lipschitz}} Constant},
  author = {Jones, D. R. and Perttunen, C. D. and Stuckman, B. E.},
  date = {1993-10-01},
  year = {1993},
  journal = {Journal of Optimization Theory and Applications},
  shortjournal = {J Optim Theory Appl},
  volume = {79},
  number = {1},
  pages = {157--181},
  issn = {1573-2878},
  url = {https://doi.org/10.1007/BF00941892},
}

@article{gablonsky2001locallybiased,
  title = {A Locally-Biased Form of the {{DIRECT}} Algorithm},
  author = {Gablonsky, J.M. and Kelley, C.T.},
  date = {2001-09-01},
  year = {2001},
  journal = {Journal of Global Optimization},
  shortjournal = {Journal of Global Optimization},
  volume = {21},
  number = {1},
  pages = {27--37},
  issn = {1573-2916},
  url = {https://doi.org/10.1023/A:1017930332101}
}

@inproceedings{wilson2020gpsampling,
  title = {Efficiently Sampling Functions from {{Gaussian}} Process Posteriors},
  booktitle = {Proceedings of the 37th {{International Conference}} on {{Machine Learning}}},
  author = {Wilson, James and Borovitskiy, Viacheslav and Terenin, Alexander and Mostowsky, Peter and Deisenroth, Marc},
  date = {2020-11-21},
  year = {2020},
  pages = {10292--10302},
  publisher = {PMLR},
  issn = {2640-3498},
  url = {https://proceedings.mlr.press/v119/wilson20a.html},
  eventtitle = {International {{Conference}} on {{Machine Learning}}}
}

@article{wilson2021pathwise,
  title = {Pathwise Conditioning of {{Gaussian}} Processes},
  author = {Wilson, James and Borovitskiy, Viacheslav and Terenin, Alexander and Mostowsky, Peter and Deisenroth, Marc},
  date = {2021},
  year = {2021},
  journal = {Journal of Machine Learning Research},
  volume = {22},
  number = {105},
  pages = {1--47},
  issn = {1533-7928},
  url = {http://jmlr.org/papers/v22/20-1260.html},
}

@book{lemieux2009montecarlo,
  title = {Monte {{Carlo}} and Quasi-{{Monte Carlo}} Sampling},
  author = {Lemieux},
  date = {2009},
  year = {2009},
  series = {Springer {{Series}} in {{Statistics}}},
  edition = {1},
  publisher = {Springer},
  location = {New York, NY},
  url = {https://doi.org/10.1007/978-0-387-78165-5},
  isbn = {978-0-387-78164-8 978-0-387-78165-5}
}

@article{jones1998ego,
   author = {Donald R. Jones and Matthias Schonlau and William J. Welch},
   url = {https://doi.org/10.1023/A:1008306431147},
   issn = {09255001},
   issue = {4},
   journal = {Journal of Global Optimization},
   month = dec,
   pages = {455-492},
   title = {Efficient Global Optimization of Expensive Black-Box Functions},
   volume = {13},
   year = {1998},
}

@InProceedings{daulton2022robustmobo,
  title = {Robust Multi-Objective {B}ayesian Optimization Under Input Noise},
  author = {Daulton, Samuel and Cakmak, Sait and Balandat, Maximilian and Osborne, Michael A. and Zhou, Enlu and Bakshy, Eytan},
  booktitle = {International Conference on Machine Learning},
  pages = {4831--4866},
  year = {2022},
  editor = {Chaudhuri, Kamalika and Jegelka, Stefanie and Song, Le and Szepesvari, Csaba and Niu, Gang and Sabato, Sivan},
  volume = {162},
  series = {Proceedings of Machine Learning Research},
  month = jul,
  publisher = {PMLR},
  url = {https://proceedings.mlr.press/v162/daulton22a.html}
}

@inproceedings{balandat2020botorch,
    author = {Balandat, Maximilian and Karrer, Brian and Jiang, Daniel and Daulton, Samuel and Letham, Ben and Wilson, Andrew G and Bakshy, Eytan},
    booktitle = {Advances in Neural Information Processing Systems},
    editor = {H. Larochelle and M. Ranzato and R. Hadsell and M.F. Balcan and H. Lin},
    pages = {21524--21538},
    publisher = {Curran Associates, Inc.},
    title = {BoTorch: A Framework for Efficient Monte-Carlo {Bayesian} Optimization},
    url = {https://proceedings.neurips.cc/paper/2020/file/f5b1b89d98b7286673128a5fb112cb9a-Paper.pdf},
    volume = {33},
    year = {2020}
}

@inproceedings{papenmeier2025exploring,
  title = {Exploring Exploration in {{Bayesian}} Optimization},
  booktitle = {Proceedings of the {{Forty-first Conference}} on {{Uncertainty}} in {{Artificial Intelligence}}},
  author = {Papenmeier, Leonard and Cheng, Nuojin and Becker, Stephen and Nardi, Luigi},
  year = 2025,
  month = jul,
  pages = {3388--3415},
  publisher = {PMLR},
  issn = {2640-3498},
  url={https://proceedings.mlr.press/v286/papenmeier25a.html},
}

@article{qing2023robust,
  title = {A Robust Multi-Objective {{Bayesian}} Optimization Framework Considering Input Uncertainty},
  author = {Qing, Jixiang and Couckuyt, Ivo and Dhaene, Tom},
  year = 2023,
  month = jul,
  journal = {Journal of Global Optimization},
  volume = {86},
  number = {3},
  pages = {693--711},
  issn = {1573-2916},
  url = {https://doi.org/10.1007/s10898-022-01262-9},
}

@inproceedings{nguyen2021valueatrisk,
  title = {Value-at-Risk Optimization with {{Gaussian}} Processes},
  booktitle = {Proceedings of the 38th {{International Conference}} on {{Machine Learning}}},
  author = {Nguyen, Quoc Phong and Dai, Zhongxiang and Low, Bryan Kian Hsiang and Jaillet, Patrick},
  year = 2021,
  month = jul,
  pages = {8063--8072},
  publisher = {PMLR},
  issn = {2640-3498},
  url = {https://proceedings.mlr.press/v139/nguyen21b.html},
}

@inproceedings{picheny2022quantile,
  title = {Bayesian Quantile and Expectile Optimisation},
  booktitle = {Proceedings of the {{Thirty-Eighth Conference}} on {{Uncertainty}} in {{Artificial Intelligence}}},
  author = {Picheny, Victor and Moss, Henry and Torossian, L{\'e}onard and Durrande, Nicolas},
  year = 2022,
  month = aug,
  pages = {1623--1633},
  publisher = {PMLR},
  issn = {2640-3498},
  url = {https://proceedings.mlr.press/v180/picheny22a.html},
}

@article{urrehman2014worstcase,
  title = {Efficient {{Kriging-based}} Robust Optimization of Unconstrained Problems},
  author = {{ur Rehman}, Samee and Langelaar, Matthijs and {van Keulen}, Fred},
  year = 2014,
  month = nov,
  journal = {Journal of Computational Science},
  volume = {5},
  number = {6},
  pages = {872--881},
  issn = {1877-7503},
  url = {https://doi.org/10.1016/j.jocs.2014.04.005},
}

@article{han2025worstcase,
  title = {Worst-Case Robust Optimization Based on an Adaptive Incremental {{Kriging}} Metamodel},
  author = {Han, Jie and Zheng, Yuxuan and Wang, Kai and Yang, Chunhua and Yuan, Xin},
  year = 2025,
  month = jan,
  journal = {Expert Systems with Applications},
  volume = {260},
  pages = {125372},
  issn = {0957-4174},
  url = {https://doi.org/10.1016/j.eswa.2024.125372},
}

@article{durham1998sequential,
  title = {A Sequential Design for Maximizing the Probability of a Favourable Response},
  author = {Durham, S.d. and Flournoy, N. and Li, W.},
  date = {1998},
  year = {1998},
  journal = {Canadian Journal of Statistics},
  volume = {26},
  number = {3},
  pages = {479--495},
  issn = {1708-945X},
  url = {https://doi.org/10.2307/3315771},
}

@article{betlei2024rladvertising,
  title = {Maximizing the Success Probability of Policy Allocations in Online Systems},
  author = {Betlei, Artem and Vladimirova, Mariia and Sebbar, Mehdi and Urien, Nicolas and Rahier, Thibaud and Heymann, Benjamin},
  date = {2024-03-24},
  year = {2024},
  journal = {Proceedings of the AAAI Conference on Artificial Intelligence},
  volume = {38},
  number = {10},
  pages = {11061--11068},
  issn = {2374-3468},
  url = {https://doi.org/10.1609/aaai.v38i10.28982}
}

@article{sun2015sss,
  title = {Fast Statistical Analysis of Rare Circuit Failure Events via Scaled-Sigma Sampling for High-Dimensional Variation Space},
  author = {Sun, Shupeng and Li, Xin and Liu, Hongzhou and Luo, Kangsheng and Gu, Ben},
  date = {2015-07},
  year = {2015},
  journal = {IEEE Transactions on Computer-Aided Design of Integrated Circuits and Systems},
  volume = {34},
  number = {7},
  pages = {1096--1109},
  issn = {1937-4151},
  url = {https://doi.org/10.1109/TCAD.2015.2404895},
}
\bibliographystyle{icml2026}

%%%%%%%%%%%%%%%%%%%%%%%%%%%%%%%%%%%%%%%%%%%%%%%%%%%%%%%%%%%%%%%%%%%%%%%%%%%%%%%
%%%%%%%%%%%%%%%%%%%%%%%%%%%%%%%%%%%%%%%%%%%%%%%%%%%%%%%%%%%%%%%%%%%%%%%%%%%%%%%
% APPENDIX
%%%%%%%%%%%%%%%%%%%%%%%%%%%%%%%%%%%%%%%%%%%%%%%%%%%%%%%%%%%%%%%%%%%%%%%%%%%%%%%
%%%%%%%%%%%%%%%%%%%%%%%%%%%%%%%%%%%%%%%%%%%%%%%%%%%%%%%%%%%%%%%%%%%%%%%%%%%%%%%
\newpage
\appendix
\onecolumn

\section{Proofs of Theoretical Results} \label[appendix]{sec:theoretical-results}
In this appendix we provide proofs of the theoretical results stated in the main text.
We restate all results before proving them for convenience.

\begin{manuallemma}{\ref*{thm:true-fail-prob-cts}}[Formal Version]
    Suppose that \(f: \sY \to \R\) is continuous and \(\vg: \sX \times \sU \to \sY\) is continuous.
    Suppose further that the distribution of \(\mathbb{P}_\rvu\) has no mass on the limit state surface or boundary of \(\sYfeas\), so that for any \(\vx \in \sX\),
    \begin{align*}
        \mathbb{P}_\rvu\bigl(f \circ \vg(\vx, \rvu) = c \bigr) &= 0, \\
        \mathbb{P}_\rvu\bigl( \vg(\vx, \rvu) \in \partial \sYfeas \bigr) &= 0.
    \end{align*}
    Then the failure probability \(P : \sX \to [0, 1]\) is continuous. If also \(\sX\) is compact then \(P\) attains its minimum value for some \(\vx^* \in \sX\).
\end{manuallemma}
\begin{proof}
Let \(\vx \in \sX\) and let \(\sS_\vx^{(1)} = \{\vu \in \sU : f \circ \vg(\vx, \vu) = c\}\) and \(\sS_\vx^{(2)} = \{\vu \in \sU : \vg(\vx, \vu) \in \partial \sYfeas\}\) denote the components of the limit state surface for \(\vx\) and let \(\sS_\vx = \sS_\vx^{(1)} \cup \sS_\vx^{(2)}\).
Then \(\mathbb{P}(\rvu \in \sU \setminus \sS_\vx) = 1\).
Observe that 
\begin{equation*}
    P(\vx) = \E_\rvu\Bigl[\mathbb{I}_{\{f \circ \vg(\vx, \rvu) \geq c\}} \mathbb{I}_{\{\vg(\vx, \rvu) \in \sYfeas\}} + \mathbb{I}_{\{\vg(\vx, \rvu) \notin \sYfeas\}}\Bigr].
\end{equation*}

Let \(\vu \in \sU \setminus \sS_\vx\) and let \((\vx_n)\) be a sequence in \(\sX\) with \(\vx_n \to \vx\) as \(n \to \infty\). Then \(f \circ \vg (\cdot, \vu) : \sX \to \R\) is continuous and \(f \circ \vg(\vx_n, \vu) \to f \circ \vg(\vx, \vu)\) as \(n \to \infty\). Further, since \(f \circ \vg(\vx, \vu) \neq c\), it follows \(\mathbb{I}_{\{f \circ \vg(\vx_n, \vu) \geq c\}}\) is eventually constant and so \(\mathbb{I}_{\{f \circ \vg(\vx_n, \vu) \geq c\}} \to \mathbb{I}_{\{f \circ \vg(\vx, \vu) \geq c\}}\).
Similarly, \(\vg(\vx, \vu) \notin \partial \sYfeas\) so by the same argument \(\mathbb{I}_{\{\vg(\vx_n, \rvu) \in \sYfeas\}} \to \mathbb{I}_{\{\vg(\vx, \rvu) \in \sYfeas\}}\).

This holds for \(\mathbb{P}_\rvu\)-almost every \(\vu \in \sU\), so by the Dominated Convergence Theorem, \(P(\vx_n) \to P(\vx)\) and so \(P\) is continuous at \(\vx\). Since \(\vx\) was arbitrary, \(P\) is everywhere continuous.

Finally, if \(\sX\) is also compact, then as a continuous function on a compact set, \(P\) attains its minimum for some \(\vx^* \in \sX\).
\end{proof}

Before proving \cref{thm:pred-fail-prob-cts}, we formally state the required regularity conditions.
\begin{assumption} \label{ass:u-no-mass-on-bndry-gp}
    \phantom{.}
    \begin{enumerate}[\alph*)]
        \item For any \(\vy_1', \dots, \vy_N' \in \sY\) (not necessarily the points chosen during optimization) and any \(v_1, \dots, v_N \in \sR\), the set \(\bigl\{\vy \in \sY \;:\; \Var\bigl[f(\vy) \,\big|\, f(\vy_1') = v_1, \dots, f(\vy_N') = v_N \bigr] = 0 \bigr\}\) has Lebesgue measure zero,
        \item For any \(\vx \in \sX\), \(\vg(\vx, \rvu)\) (where \(\rvu \sim \mathbb{P}_\rvu\)) is absolutely continuous with respect to Lebesgue measure on \(\sY\),
        \item For any \(\vx \in \sX\), \(\mathbb{P}_\rvu\bigl( \vg(\vx, \rvu) \in \partial \sYfeas \bigr) = 0\).
    \end{enumerate}
\end{assumption}
The first of these assumptions ensures that the GP is sufficiently non-degenerate. Combined with the second assumption, this will be sufficient to deduce that \(\mathbb{P}(f \circ \vg(\vx, \rvu) = c \,|\, \mathcal{F}_n) = 0\). That is, that \(\mathbb{P}_\rvu\) and \(f\) jointly place no mass on the, now stochastic, limit state surface. Combined with the final part of \cref{ass:u-no-mass-on-bndry-gp} these conditions are analogous to those in \cref{thm:true-fail-prob-cts}.

\begin{manuallemma}{\ref*{thm:pred-fail-prob-cts}}[Formal Version]
    Suppose that \(f\) is a sample-continuous Gaussian process, that \(\vg : \sX \times \sU \to \sY\) is continuous and that \cref{ass:u-no-mass-on-bndry-gp} holds. Then \(P_n: \sX \to [0, 1]\) is continuous. If also \(\sX\) is compact then \(P_n\) attains its minimum value for some \(\vx^*_n \in \sX\).
\end{manuallemma}
\begin{proof}
    Write \((\Omega, \mathcal{F}, \mathbb{P})\) for the underlying probability space.
    
    We will first show that, for all \(\vx \in \sX\), \(\mathbb{P}(f \circ \vg(\vx, \rvu) = c \,|\, \mathcal{F}_n) = 0\) almost surely.
    Indeed, the first two parts of \cref{ass:u-no-mass-on-bndry-gp} give that for all \(\vx \in \sX\), \(\mathrm{Var}[ f \circ \vg(\vx, \rvu) \,|\, \rvu, \mathcal{F}_n] > 0\) almost surely.
    Thus, \(f \circ \vg(\vx, \rvu) \,|\, \rvu, \mathcal{F}_n\) is an almost surely non-degenerate Gaussian random variable and hence \({\mathbb{P}_\rvu(f \circ \vg(\vx, \rvu) = c \,|\, \rvu, \mathcal{F}_n) = 0}\) almost surely.
    Therefore,
    \[\mathbb{P}\bigl( f \circ \vg(\vx, \rvu) = c \,\big|\, \mathcal{F}_n \bigr) = \E\bigl[ \mathbb{P}\bigl( f \circ \vg(\vx, \rvu) = c \,\big|\, \rvu, \mathcal{F}_n \bigr) \,\big|\, \mathcal{F}_n\bigr] = 0 \quad \text{a.s.}\]

    The remainder of the proof follows that of \cref{thm:true-fail-prob-cts}.
    
    Let \(\vx \in \sX\), and define events \(A = \{\omega \in \Omega : f(\vg(\vx, \rvu(\omega); \omega) \geq c\}\) and \(B = \{\omega \in \Omega : \vg(\vx, \rvu(\omega)) \notin \sYfeas\}\).
    These are indeed events because \(f\) is sample-continuous (and so jointly measurable) and \(\vg\) is continuous.
    Then \({P_n(\vx) = \mathbb{E}[\mathbb{I}_A (1 - \mathbb{I}_B) + \mathbb{I}_B \,|\, \mathcal{F}_n]}\) is also well-defined.
    
    Let \((\vx_m)\) be a sequence in \(\sX\) with \(\vx_m \to \vx\) as \(m \to \infty\) and define events \(A_m\) and \(B_m\) as for \(A\) and \(B\) but with \(\vx_m\) in place of \(\vx\).
    That is, \(A_m = \{\omega \in \Omega : f(\vg(\vx_m, \rvu(\omega); \omega) \geq c\}\), \(B_m = \{\omega \in \Omega : \vg(\vx_m, \rvu(\omega)) \notin \sYfeas\}\) and \(P_n(\vx_m) = \mathbb{E}[\mathbb{I}_{A_m} (1 - \mathbb{I}_{B_m}) + \mathbb{I}_{B_m} \,|\, \mathcal{F}_n]\).

    Since \(\vg\) is continuous, we conclude \(f (\vg(\cdot, \rvu))\) has almost surely continuous sample paths.
    Further, \({\mathbb{P}(f(\vg(\vx, \rvu)) = c \,|\, \mathcal{F}_n) = 0}\) and so \(\mathbb{I}_{A_m}\) is almost surely eventually constant, and hence \(\mathbb{I}_{A_m} \to \mathbb{I}_A\) as \(m \to \infty\).
    A similar argument establishes that \(\mathbb{I}_{B_m} \to \mathbb{I}_B\) almost surely, and so the Dominated Convergence Theorem gives \(P_n(\vx_m) \to P_n(\vx)\) as \(m \to \infty\).

    Finally, if \(\sX\) is also compact, then as a continuous function on a compact set, \(P_n\) attains its minimum for some \(\vx^*_n \in \sX\).
\end{proof}

\thmkgnonneg*
\begin{proof}
    For all \(\vx' \in \sX\)
    \begin{align*}
        \max_{\vx \in \sX} R_{n+1}(\vx) &\geq R_{n+1}(\vx') \\
        \Rightarrow \quad \E\biggl[\max_{\vx \in \sX} R_{n+1}(\vx) \,\bigg|\, \mathcal{F}_n, \vy_{n+1} = \vy \biggr] &\geq \E\bigl[ R_{n+1}(\vx') \,\big|\, \mathcal{F}_n, \vy_{n+1} = \vy \bigr] \geq R_n(\vx')
    \end{align*}
    where the last inequality follows from Jensen's inequality because \(z \mapsto -\log z\) is convex.
    This holds for all \(\vx'\) and so
    \begin{gather*}
        \E\biggl[\max_{\vx \in \sX} R_{n+1}(\vx) \,\bigg|\, \mathcal{F}_n, \vy_{n+1} = \vy \biggr] \geq \max_{\vx \in \sX} R_n(\vx) \\
        \Rightarrow\quad \alpha^\text{KG-MR}_n(\vy) = \E\biggl[\max_{\vx \in \sX} R_{n+1}(\vx) \,\bigg|\, \mathcal{F}_n, \vy_{n+1} = \vy \biggr] - \max_{\vx \in \sX} R_n(\vx) \geq 0.
    \end{gather*}
\end{proof}

\section{Optimization of the One-Shot Knowledge Gradient Approximations} \label[appendix]{sec:kg-optimization}
We optimize the one-shot KG-MR acquisition function~\eqref{eq:kg-oneshot} using multi-start L-BFGS-B.
Care must be taken when choosing starting locations for the multi-start optimization since the optimization landscape contains many local minima.
Indeed, suppose that \(\hat{R}_{n+1}(\vx; \vy, z_j)\) has \(m\) local minima in \(\vx\). Then the maximand on the right-hand side of~\eqref{eq:kg-oneshot} has \(m^{N_v}\) local minima in \(\vx_1, \dots, \vx_{N_v}\) each with its own optimal \(\vy\).

In our experiments we adopt the following procedure to choose initial points. First, we evaluate the discrete KG-MR~\eqref{eq:kg-discrete} at a large number \(N_\mathrm{raw}\) of \(\vy\)-values. We then choose \(N_\mathrm{restarts}\) of these using Boltzmann sampling to randomly select starting locations which favor those with higher values. If the best value was not chosen, we replace the last sampled value with the best value. These values of \(\vy\) and their corresponding \(\vx_1, \dots, \vx_{N_v} \in \sXd\) are used to initialize the multi-start L-BFGS-B optimization of~\eqref{eq:kg-oneshot}. This implementation of Boltzmann sampling is standard within BoTorch.

\section{Baseline Algorithms} \label[appendix]{sec:baseline-algos}
In this appendix, we describe the algorithms which exist in the literature which were used as baselines.

\subsection{The Method of \texorpdfstring{\citet{huang2010egoReliability}}{Huang and Chan (2010)}}
\citet{huang2010egoReliability} propose a Bayesian optimization method which chooses between four acquisition functions. If there are currently no feasible samples, then the probability of feasibility~\eqref{eq:huang-chan-f} is optimized. Otherwise, three other acquisition functions are tried in order until one yields a point at least a distance \(\varepsilon_s\) from all existing samples.
First~\eqref{eq:huang-chan-ls} is tried which seeks to maximize distance to existing observations while being within \(\Delta\) of the predicted limit state surface.
If that fails, \eqref{eq:huang-chan-tn} is maximized to seek new, disconnected areas of the limit state surface, known as tunneling.
Finally, if all else fails, \eqref{eq:huang-chan-mv} is optimized to seek the point with maximum variance and hence reduce uncertainty in the model.
Listed together, the acquisition functions are,
\begin{subequations}\label{eq:huang-chan}
\begin{align}
    \alpha_n^\mathrm{F}(\vy) &= \mathbb{P}(f(\vy) \leq c) = \Phi{\biggl(\frac{c - \mu_n(\vy)}{\sqrt{k_n(\vy, \vy)}}\biggr)}, \label{eq:huang-chan-f}\\
    \alpha_n^\mathrm{LS}(\vy) &= \mathbb{I}_{\{|\mu_n(\vy) - c| \leq \Delta\}} \, \min_{i=1,\dots,n} \frac{\|\vy_i - \vy\|}{\|\vb - \va\|}, \label{eq:huang-chan-ls} \\
    \alpha_n^\mathrm{TN}(\vy) &= \alpha_n^\mathrm{F}(\vy) \min_{\substack{i=1,\dots,n \\ \text{s.t. } f(\vy_i) \leq c}}{\!\!\frac{\|\vy_i - \vy\|}{\|\vb - \va\|}}, \label{eq:huang-chan-tn} \\
    \alpha^\mathrm{MV}_n(\vy) &= \sqrt{k_n(\vy, \vy)}. \label{eq:huang-chan-mv}
\end{align}
\end{subequations}
As before, here \(\vy_1, \dots, \vy_n\) are the locations of the \(n\) observations made so far, and \(\va, \vb\) define the bounds of the problem such that \(\sYfeas = \prod_{j=1}^{d_y} [a_j, b_j]\).
Note that in~\eqref{eq:huang-chan} we have kept the names from \citep{huang2010egoReliability}, but \(\alpha^\mathrm{MV}_n\) is a different acquisition function to the one used for TS-MR defined in \cref{sec:thompson-sampling}.

We optimize the limit state exploration function \(\alpha_n^\mathrm{LS}\) from~\eqref{eq:huang-chan-ls} using DIRECT \citep{jones1993lipschitzian,gablonsky2001locallybiased} and all other acquisition functions using multi-start L-BFGS-B.
For numerical stability, we optimize the logarithm of~\eqref{eq:huang-chan-f} rather than the acquisition function itself.
In \citep{huang2010egoReliability}, the GP hyperparameters are fitted using a variogram, however for a fair comparison with the other algorithms, we use the same GP model and maximum a posteriori~(MAP) estimation method for the hyperparameters for all algorithms.
Rather than implementing the termination criteria from this work, we instead run the algorithm for the same fixed evaluation budget as the other algorithms for a fair comparison.
In the experimental section, we denote this algorithm HC.

\subsection{The Method of \texorpdfstring{\citet{bichon2008egra}}{Bichon et al. (2008)}}
\label[appendix]{sec:baseline-algos-egra}
\citet{bichon2012efficient} propose a family of methods to tackle three problems in reliability-based design optimization, which include the problem posed in~\eqref{eq:problem-statement}.
The work extends their earlier work in active learning in which they proposed efficient global reliability analysis (EGRA) to efficiently estimate the failure probability \citep{bichon2008egra}.
They present three possible ways to combine BO with EGRA:
\begin{enumerate}[nosep]
    \item A nested approach using a separate surrogate for the function \(\vu \mapsto f(\vg(\vx, \vu))\) for each \(\vx \in \sX\) visited,
    \item A nested approach using a single surrogate to model \(h(\vx, \vu) = f(\vg(\vx, \vu))\),
    \item A sequential approach which alternates BO to find a good \(\vx\) and EGRA to determine the limit state surface.
\end{enumerate}
Their general formulation of the problem does not assume the composition \(f(\vg(\vx, \rvu))\) meaning the limit state surface could be different for different \(\vx\).
However, in our case, the surface is essentially the same for each \(\vx\) after an appropriate transform depending on \(\vg\).
\citet{bichon2012efficient} include such a problem in section 5.3 (steel column), and remark that the algorithms they propose are essentially identical in this case.
Further, their formulation of the problem is a constrained optimization, with deterministic constraints coming from the original objective.
At the first design point, a full EGRA analysis is performed, and subsequent design points are used to search the design space for points which satisfy these constraints using constrained BO.
However, in the problem considered in this paper, \eqref{eq:problem-statement} contains no constraints from an original objective.
Therefore, to the best of our understanding, in the setting of~\eqref{eq:problem-statement} where there are no black-box constraints, the algorithm proposed in \citep{bichon2012efficient} reduces to EGRA.

The EGRA algorithm was inspired by EI \citep{jones1998ego} and can be summarized as maximizing
\begin{equation} \label{eq:bichon-egra}
    \alpha_n^\mathrm{EGRA}(\vy) = \E\Bigl[ \max\Bigl\{\epsilon - \bigl| c - f(\vy) \bigr|,\, 0\Bigr\} \,\Big|\, \mathcal{F}_n \Bigr]
\end{equation}
where \(\epsilon = \kappa \sqrt{k_n(\vy, \vy)}\) is proportional to the marginal posterior standard deviation of \(f\).
Like HC, it explores the whole limit state surface, rather than concentrating on the parts contributing most to the failure probability at the point with maximal reliability.

To ensure a meaningful comparison, in the experiments we use the same GP model for all algorithms, with MAP estimated hyperparameters, rather than using either of the models from \citep{bichon2008egra} or \citep{bichon2012efficient}.
For the value of \(\kappa\), we use \(\kappa = 2\) as recommended in \citet{bichon2008egra,bichon2012efficient}. \citet{bect2011probFailure} found a value of \(\kappa = 0.5\) to give better results for the estimation of failure probabilities, however for maximizing reliability we found little difference, as is shown in \cref{fig:results-rare-egra-sensitivity}.

\begin{figure}[htbp]
    \centering
    \includegraphics[width=\linewidth]{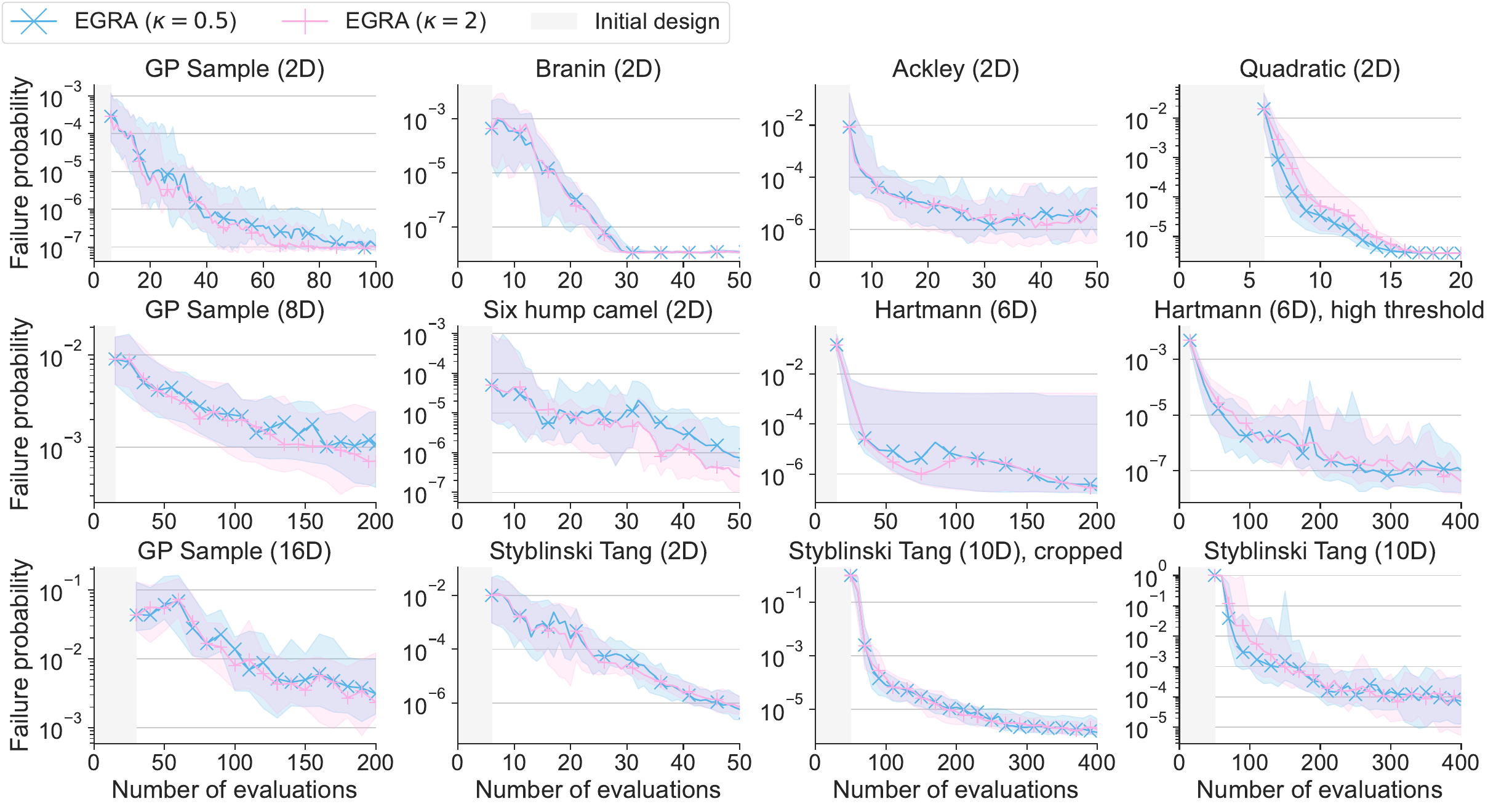}
    \caption[EGRA algorithm sensitivity study.]{Failure probabilities of the two EGRA variants for the 12 test problems in \cref{fig:results-rare}. The failure probability associated with the recommended solution is shown as a function of number of evaluations of the expensive black-box function. The solid lines show the median failure probability over 30 repeats and the shaded regions show the upper and lower quartiles.}
    \label{fig:results-rare-egra-sensitivity}
\end{figure}

\section{Further Experimental Details} \label[appendix]{sec:further-experimental-details}
In this appendix, we give further details on the experimental procedure which is important for the reproduction of results but can be safely omitted on first reading.

\subsection{Indicator Smoothing} \label[appendix]{sec:further-experimental-details--indicator-smoothing}
In \cref{sec:approximations} we briefly introduced a smooth approximation to \(\mathbb{I}_{\{\vg(\vx, \vu) \in \sYfeas\}}\) used to resolve the lack of differentiability in~\eqref{eq:pn-naive-approx}.

We wish to avoid evaluating \(f\) outside \(\sYfeas\) because, for example, the black-box function may be given by a simulation which raises an error or gives outputs which violate the GP modeling assumption outside this region. Therefore, we use an asymmetric smoothing function \(\iota(\vg(\vx, \vu); \delta) \approx \mathbb{I}_{\{\vg(\vx, \vu) \in \sYfeas\}}\) which has a value of zero when \(\vg(\vx, \vu) \notin \sYfeas\) and builds smoothly inside the feasible region.
Here \(\delta > 0\) controls the width of the interval around the boundary affected by the smoothing.
\cref{fig:bounds-smoothing} shows the contribution of this approximation to the probability of failure in one dimension.

\begin{figure}[htb]
    \centering
    \includegraphics[width=0.6\linewidth]{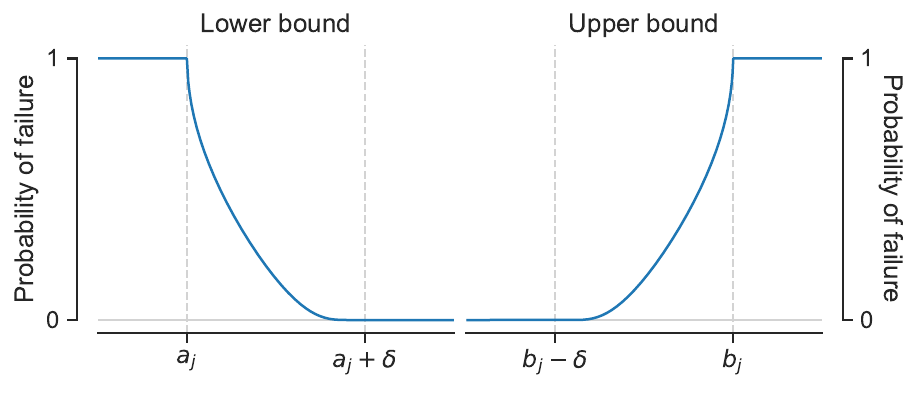}
    \caption{The approximation \(1 - \iota(\vg(\vx, \vu); \delta)\) of \(\mathbb{I}_{\{\vg(\vx, \vu) \notin \sYfeas\}}\) is shown as a function of \(\vg(\vx, \vu)\). This is defined in~\eqref{eq:bounds-indicator-smoothing} and used to smooth discontinuities in~\eqref{eq:pn-naive-approx}.}
    \label{fig:bounds-smoothing}
\end{figure}

The precise formula used for the bounds smoothing when \(\sYfeas = \prod_{j=1}^{d_y} [a_j, b_j]\) is a hyperrectangle is
\begin{equation} \label{eq:bounds-indicator-smoothing}
    \iota(\vg(\vx, \vu); \delta) = \prod_{j=1}^{d_y} G\biggl(\frac{g_j(\vx, \vu) - a_j}{\delta}\biggr) G\biggl(\frac{b_j - g_j(\vx, \vu)}{\delta}\biggr)
\end{equation}
where \(\delta\) is a smoothing parameter controlling the width of the smoothing region, and where \(g_j\) is the function evaluating the \(j\)th component of \(\vg\). For \(z \in (0, 1)\), the function \(G(z) = F(z / 1-z)\) where \(F(\cdot)\) is the cumulative density function of a Gamma distribution with shape parameter \(0.5\) and rate parameter \(1\). For \(z \leq 0\) we define \(G(z) = 0\) and for \(z \geq 1\) we set \(G(z) = 1\).
This formula comes from considering the \(j\)th lower bound \(a_j\) to be a random variable \(A_j = a_j + \delta \frac{Z_j}{1 + Z_j}\) where \(Z_j \sim \mathrm{Gamma}(0.5, 1)\), and similarly for the upper bounds.

In practice, we keep \(\delta\) very small, using a value of \(\delta = \min(0.05 \ell_\text{min}, 0.1)\) where \(\ell_\text{min}\) is the minimum side-length of \(\sYfeas\).

\subsection{Model Parameters}\label[appendix]{sec:further-experimental-details--model-params}
For priors on the GP hyperparameters, we use \(s^2 \sim \mathrm{Gamma}(2, 0.15)\) for the output scale and \(\ell_1, \dots, \ell_d \sim \mathrm{Gamma}(3, 10)\). Here, gamma distributions are parametrized using the shape and rate convention. The constant mean has no prior (i.e. an improper prior) and variance of the observation noise is fixed to \(\sigma^2 = 0.01^2\) for numerical stability. Actual observations \(v_i = f(\vy_i)\) have no noise added.

The GP hyperparameters are refitted every time a new observation is collected.
Before fitting, we normalize and standardize the observations so that the inputs lie in \([0, 1]^{d_y}\) and the outputs have zero mean and unit variance. When making predictions with the GP, the inputs and outputs are transformed and untransformed appropriately. This is standard in BoTorch.

\subsection{Optimizing the Recommendations} \label[appendix]{sec:further-experimental-details--optimizing-recommendations}
Recommendations \(\vx^*_n\) (\(n = 1, 2, \dots\)) are generated from the GP surrogate using~\eqref{eq:recommendations-approx} for all algorithms.
For the 2D problems, 10 restarts are chosen by evaluating the posterior mean on 1024 Sobol' points and using Boltzmann sampling, manually ensuring that the best found value is included as a start if necessary.
These starting values are then optimized using L-BFGS-B.
A qMC estimate with \(N_u = 1024\) is used in~\eqref{eq:recommendations-approx-mcapprox} and the bounds smoothing parameter \(\delta\) in~\eqref{eq:recommendations-approx-jhat} is set to \(\delta = \min(0.05 \ell, 0.1)\) where \(\ell\) is the smallest side length in \(\sYfeas\).
A small nugget is added to the GP variance for numerical stability.

For the higher-dimensional problems (i.e. greater than 2D) the same is done, but the best value found is then fine-tuned with L-BFGS-B using a \(N_u = 131,072\), since these problems benefit empirically from a larger value of \(N_u\).

\subsection{Test Problems}\label[appendix]{sec:further-experimental-details--test-problems}
The test problems used, along with the threshold \(c\) and perturbation distribution \(\mathbb{P}_\rvu\) are summarized in this section. All problems use additive perturbations, \(\vg(\vx, \vu) = \vx + \vu\).

\subsubsection*{Branin (2D)}
The feasible domain is \(\sYfeas = [-5, 10] \times[0, 15]\), the threshold is \(c=60\) and the formula is
\begin{equation}
    f(\vy) = \left(y_2 - \frac{5.1}{4 \pi^2} y_1^2 + \frac{5}{\pi} y_1 - 6\right)^2 + 10\left(1-\frac{1}{8\pi}\right) \cos(y_1) + 10.
\end{equation}
The perturbation distribution is \(\mathbb{P}_\rvu = \mathcal{N}\bigl(0, \mathrm{diag}(0.8, 0.8)^2 \bigr)\).

\subsubsection*{Six Hump Camel (2D)}
The feasible domain is \(\sYfeas = [-3, 3] \times [-2, 2]\), the threshold is \(c=2\) and the formula is
\begin{equation}
    f(\vy) = \left(4 - 2.1 y_1^2 + \frac{1}{3}y_1^4\right)y_1^2 + y_1 y_2 + 4\bigl(y_2^2 - 1\bigr) y_2^2.
\end{equation}
The perturbation distribution is \(\mathbb{P}_\rvu = \mathcal{N}\bigl(0, \mathrm{diag}(0.2, 0.1)^2\bigr)\).

\subsubsection*{Ackley (2D)}
The feasible domain is \(\sYfeas = [-32.768, 32.768]^2\), the threshold is \(c=20.5\) and the formula is
\begin{equation}
    f(\vy) = -20 \exp\left(-0.2 \sqrt{\frac{1}{2}\bigl(y_1^2 + y_2^2\bigr)}\right) - \exp\left( \frac{1}{2} \bigl[\cos(2 \pi y_1) +  \cos(2 \pi y_2) \bigr] \right) + 20 + e.
\end{equation}
The perturbation distribution is \(\mathbb{P}_\rvu = \mathcal{N}\bigl(0, \mathrm{diag}(3, 3)^2)\).

\subsubsection*{Quadratic (2D)}
The feasible domain is \(\sYfeas = [0, 1]^2\), the threshold is \(c=0.09\) and the formula is
\begin{equation}
    f(\vy) = (y_1 - 0.3)^2 + (y_2 - 0.3)^2.
\end{equation}
The perturbation distribution is \(\mathcal{N}\bigl(0, \mathrm{diag}(0.06, 0.06)^2\bigr)\).

\subsubsection*{Styblinski-Tang (2D, 10D)}
Writing \(d = 2, 10 \) for the dimension, the feasible domain is \(\sYfeas = [-5, 5]^d\). The formula is
\begin{equation}
    f(\vy) = \frac{1}{2} \sum_{i=1}^d \bigl(y_i^4 - 16 y_i^2 + 5 y_i \bigr).
\end{equation}
For the 2D problem, the threshold is \(c = -20\) and the perturbation distribution is \(\mathbb{P}_\rvu = \mathcal{N}\bigl(0, \mathrm{diag}(0.25, 0.5)^2\bigr)\).
For the 10D problem, the threshold is \(c = -300\) and the perturbation distribution is \(\mathbb{P}_\rvu = \mathcal{N}\bigl(0, \mathrm{diag}(0.4, 0.4, 0.4, 0.1, \dots, 0.1)^2 \bigr)\).
For the cropped version of the 10D problem, the feasible domain is reduced to \(\sYfeas = [-5, 0] \times [-5, 5]^3 \times [-5, 0]^6\).

\subsubsection*{Hartmann (6D)}
The feasible domain is \(\sYfeas = [0, 1]^6\) and the formula is
\begin{equation}
    f(\vy) = - \sum_{i=1}^4 \alpha_i \exp\left( - \sum_{j=1}^6 A_{ij} (y_j - P_{ij})^2 \right)
\end{equation}
where
\begin{gather*}
    \alpha = \begin{pmatrix}1 & 1.2 & 3 & 3.2\end{pmatrix}^T \\
    \mA = \begin{pmatrix}
        10 & 3 & 17 & 3.5 & 1.7 & 8 \\
        0.05 & 10 & 17 & 0.1 & 8 & 14 \\
        3 & 3.5 & 1.7 & 10 & 17 & 8 \\
        17 & 8 & 0.05 & 10 & 0.1 & 14
    \end{pmatrix} \\
    \mP = \begin{pmatrix}
        0.1312 & 0.1696 & 0.5569 & 0.0124 & 0.8283 & 0.5886 \\
        0.2329 & 0.4135 & 0.8307 & 0.3736 & 0.1004 & 0.9991 \\
        0.2348 & 0.1451 & 0.3522 & 0.2883 & 0.3047 & 0.6650 \\
        0.4047 & 0.8828 & 0.8732 & 0.5743 & 0.1091 & 0.0381
    \end{pmatrix}.
\end{gather*}
For the standard version, the threshold is \(c = -1\) and the perturbation distribution is \(\mathbb{P}_\rvu = \mathcal{N}\bigl(0, \mathrm{diag}(0.05, \dots, 0.05)^2\bigr)\).
For the version with the high threshold, we use \(c = -0.05\) with \(\mathbb{P}_\rvu = \mathcal{N}\bigl(0, \mathrm{diag}(0.07, \dots, 0.07)^2\bigr)\).
The perturbation distribution is \(\mathbb{P}_\rvu = \mathcal{N}\bigl(0, \mathrm{diag}(0.05, \dots, 0.05)^2\bigr)\).

\subsubsection*{GP Sample (2D, 8D, 16D)}
The GP sample test problems were generated using a random Fourier feature approximation with 1024 Fourier features.
The GP has zero mean and a Mat\'ern-5/2 kernel.
The output scale was 10 and the length scales were approximately \(0.2 \sqrt{d}\) to maintain the problem complexity with increasing dimension.
The thresholds were set so that 33\% of the volume of \(\sYfeas\) is infeasible.
The perturbation distribution was chosen as an isotropic Gaussian with scale chosen to achieve an appropriate optimal failure probability.
The absolute values used are summarized in \cref{tab:gp-problem-parameters}.

The feasible domain is \(\sYfeas = [0, 1]^d\).

\begin{table}[h]
    \centering
    \caption{Parameters used to generate the GP test problems.}
    \label{tab:gp-problem-parameters}
    \begin{tabular}{rrrrl}
    \toprule
        Dimension & Isotropic length scale & Threshold, \(c\) & Failure volume & Perturbation distribution, \(\mathbb{P}_\rvu\) \\
    \midrule
        2 & 0.28 & 3.6 & 33\% & \(\mathcal{N}\bigl(0, 0.04^2 \mI \bigr)\) \\
        8 & 0.57 & 1.2 & 33\% & \(\mathcal{N}\bigl(0, 0.06^2 \mI \bigr)\) \\
        16 & 0.8 & 0.6 & 33\% & \(\mathcal{N}\bigl(0, 0.07^2 \mI \bigr)\) \\
    \bottomrule
    \end{tabular}
\end{table}

\subsection{Algorithm Parameters} \label[appendix]{sec:further-experimental-details--algorithm-parameters}
For all experiments, the number of fantasies used in osKG-MR and dKG-MR is \(N_v = 64\).
The number of qMC points used for the expectation over \(\rvu\) in both KG-MR algorithms and in TS-MR is \(N_u = 64\).
For dKG-MR the size of the discretization \(\sXd\) is \(N_x = 512\).
Where importance sampling is used (i.e. for experiments with extreme failure probabilities), we use a scale factor of \(\tau = 3\).
The bounds smoothing parameter \(\delta\) was set to \(\min(0.05 \ell_\text{min}, 0.1)\) where \(\ell_\text{min}\) is the smallest side length of the feasible domain \(\sYfeas\).
The threshold smoothing parameter in TS-MR was set to \(\rho = 0.01\).
For EGRA, we use \(\kappa = 2\) for most experiments, but \(\kappa = 0.5\) in the sensitivity study in \cref{fig:results-rare-egra-sensitivity} in \cref{sec:baseline-algos-egra}.
Multi-start L-BFGS-B uses 10 restarts for the optimization of the acquisition function.

No guidance is given in \citep{huang2010egoReliability} for how to set the parameters \(\varepsilon_s\) and \(\Delta\) in the HC algorithm.
We choose the minimum distance between points according to \(\varepsilon_s = 0.01 \ell_\mathrm{diagonal}\) where \(\ell_\mathrm{diagonal}\) is the length of the diagonal of the feasible box \(\sYfeas\).
This allows the minimum distance to grow with the dimension of the problem, since higher-dimensional problems naturally contain more volume and larger distances.
It also gives values in line with those used in \citep{huang2010egoReliability} where this value is reported.
For the parameter \(\Delta\) describing the width of the region of interest around the limit state surface, no values are reported in \citep{huang2010egoReliability} and so we adopt a best-effort approach.
Being a measure in objective space, intuitively the value of \(\Delta\) used should scale with the gradient \(\|\nabla f\|\) around the limit state surface.
For the 2D problems, we plotted the landscape and chose a value of \(\Delta\) to give a thin band around the limit state surface based on the contours.
For higher-dimensional problems, we chose an appropriate value using either the formula (in the case of Hartmann and Styblinski-Tang) or the histogram of function values (in the case of the GP problems).
Of course, this choice method would not be possible for a real black-box problem, and conveys an unfair advantage on the HC algorithm in our experiments.
However, since HC is largely beaten by KG-MR and TS-MR, this does not invalidate the conclusions of the comparison.
The parameters used for HC and the size \(n_0\) of the initial design used by all algorithms are given in \cref{tab:alg-parameters}.

The parameters used when generating recommendations are specified in \cref{sec:further-experimental-details--optimizing-recommendations}.

\begin{table}[h]
    \centering
    \caption{Experiment-specific parameters for the algorithms. For HC, these are the minimum distance between samples \(\varepsilon_s\) and the distance in objective space \(\Delta\) defining the band of interest around the limit state surface. Additionally, the size \(n_0\) of the initial design is used by all algorithms.}
    \label{tab:alg-parameters}
    \begin{tabular}{lrrr}
    \toprule
        & & \multicolumn{2}{c}{HC} \\
        \cmidrule(lr){3-4}
        Experiment & \(n_0\) & \(\varepsilon_s\) & \(\Delta\) \\
    \midrule
        GP (2D) & \(6\) & 0.014 & 0.6\\
        GP (8D) & \(15\) & 0.028 & 0.6\\
        GP (16D) & \(30\) & 0.04 & 0.6 \\
        Branin (2D) & \(6\) & 0.21 & 10 \\
        Six hump camel (2D) & \(6\) & 0.072 & 0.4 \\
        Styblinski-Tang (2D) & \(6\) & 0.14 & 10 \\
        Ackley (2D) & \(6\) & 0.93 & 0.2 \\
        Quadratic (2D) & \(6\) & 0.014 & 0.01 \\
        Hartmann (6D) & \(15\) & 0.024 & 0.02 \\
        Hartmann (6D), high threshold & \(15\) & 0.024 & 0.02 \\
        Styblinski-Tang (10D) & \(50\) & 0.32 & 10 \\
        Styblinski-Tang (10D), cropped & \(50\) & 0.22 & 10 \\
    \bottomrule
    \end{tabular}
\end{table}

\section{Example Query Patterns} \label[appendix]{sec:example-query-patterns}
\begin{figure}[p]
    \centering
    \includegraphics[width=\linewidth]{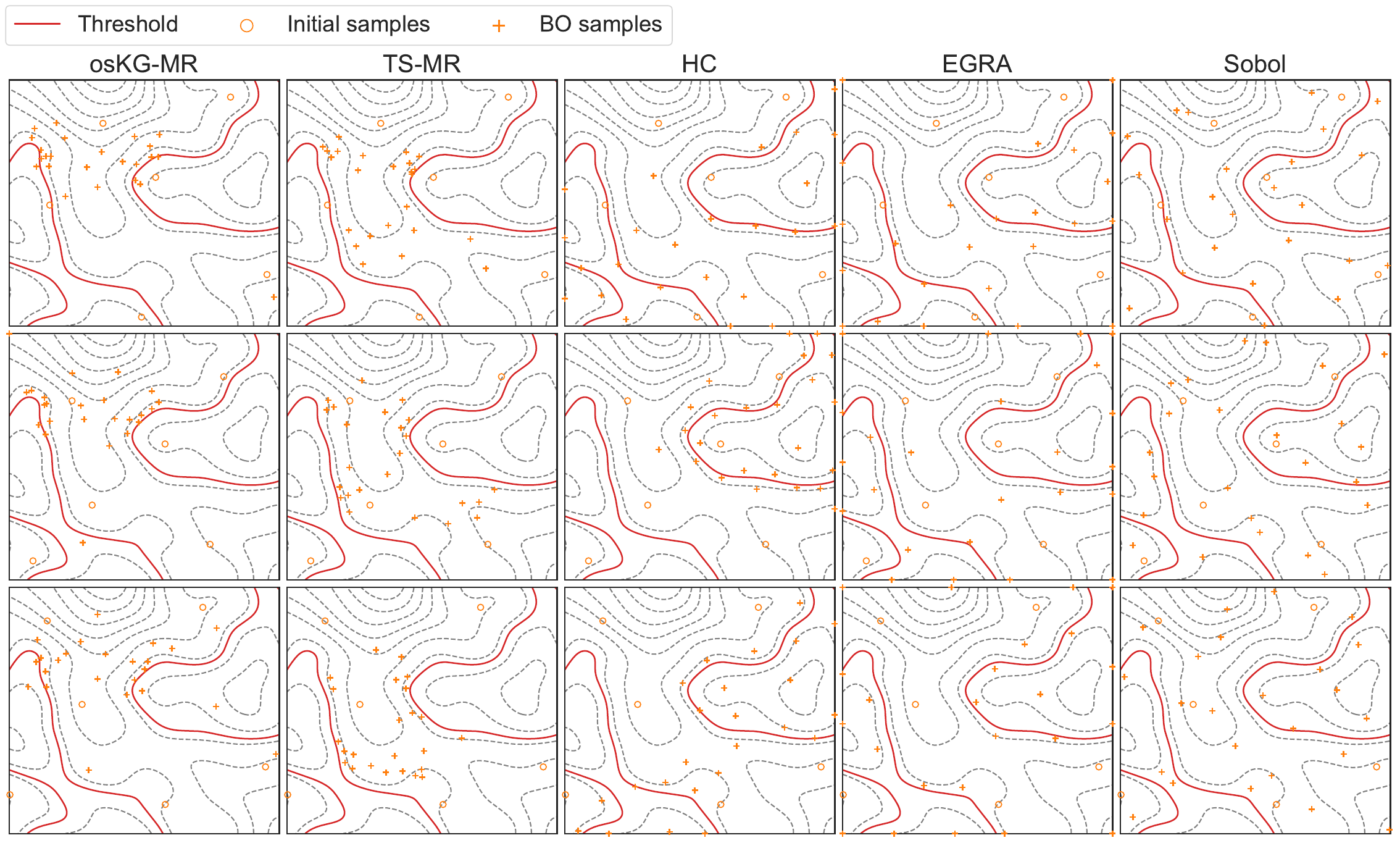}
    \caption{The first 30 sampling locations for five algorithms tested on the 2D GP test problem. The rows show three different runs, corresponding to three different initial designs.}
    \label{fig:example-query-patterns--gp2d}

    \vspace{\floatsep}

    \includegraphics[width=\linewidth]{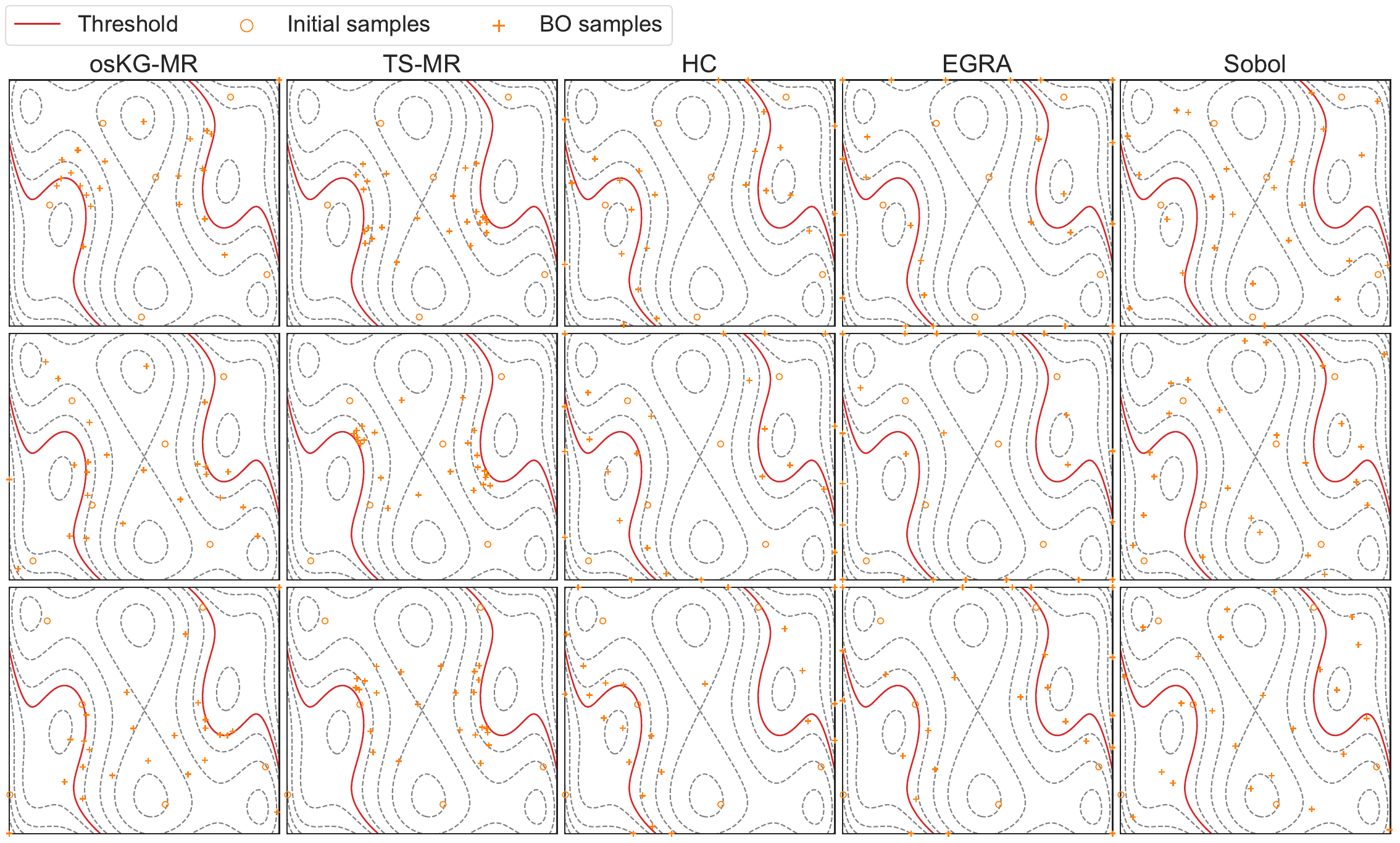}
    \caption{The first 30 sampling locations for five algorithms tested on the six hump camel test problem. The rows show three different runs, corresponding to three different initial designs.}
    \label{fig:example-query-patterns--six-hump-camel}
\end{figure}

\Cref{fig:example-query-patterns--gp2d,fig:example-query-patterns--six-hump-camel} show the first 30 samples collected by the five algorithms tested on the 2D GP and six hump camel test problems.
The osKG-MR and TS-MR algorithms clearly focus on the relevant part of the limit state surface, while the HC and EGRA algorithms sometimes initially focus on the wrong parts.
In one case, HC fails to discover all disconnected parts of the limit state surface.
The HC and EGRA algorithms also attempt to sample on the boundary \(\partial \sYfeas\) in places where the value of \(f\) is close to the threshold \(c\).
This phenomenon been observed in other works on robust Bayesian optimization \citep{qing2023robust} where the algorithm cannot efficiently reduce the uncertainty in the GP outside the boundary.
KG-MR and TS-MR naturally incorporate knowledge that perturbed solutions \(\vy \in \sY \setminus \sYfeas\) beyond the boundary will be infeasible and so do not suffer from this.

\section{Acquisition Function Timings} \label[appendix]{sec:acquisition-timings}
It is important when choosing an acquisition function to be aware of the time required to optimize the acquisition function, since this should be small compared with the time or cost required to evaluate the black-box function \(f\).
\Cref{tab:timings-2d,tab:timings-highdim} show the lower and upper quartiles of the time required to optimize the acquisition function for the problems in the middle two columns of \cref{fig:results-rare}.
These are not presented as a means to claim one algorithm is better than another, but rather to give an idea for how expensive the expensive black-box function \(f\) should be.
For example, if evaluations of \(f\) take of the order of an hour then spending a 30s to optimize the one-shot KG-MR for the cropped, 10D Styblinski-Tang problem would be negligible.

The KG-MR algorithms, and in particular the one-shot KG-MR, are computationally intensive, and thus for problems with \(d > 2\) they are run using a GPU while all other algorithm-problem combinations are run only using CPUs. The GPU used is an NVIDIA L40 (48GB RAM), and is paired with 10 cores from an AMD 9354P Genoa (Zen 4) processor and 59GB RAM. For the CPU only experiments, 8 cores of an Intel Xeon Platinum 8268 (Cascade Lake) processor are used with 19.6GB RAM.

\begin{table}[htp]
    \centering
    \caption{Lower and upper quartiles of the time required to optimize the acquisition functions on the 2D test problems from the middle two columns of \cref{fig:results-rare} at several points during the optimization. The value of \(n\) denotes the number of the sample currently being optimized. On the 2D test problems, all algorithms were run using 8 CPUs and a total of 29.6GB RAM.}
    \label{tab:timings-2d}
    \begin{tabular}{llrrrr}
    \toprule
        & & \thead{Ackley (2D)} & \thead{Branin (2D)} & \thead{Six hump\\camel (2D)} & \thead{Styblinski\\Tang (2D)} \\
    \midrule
        Discrete KG-MR & $n=7$ & 19s-23s & 16s-19s & 15s-19s & 17s-19s \\
         & $n=50$ & 25s-33s & 9s-65s & 7s-14s & 8s-30s \\
    \cline{1-6}
        One-shot KG-MR & $n=7$ & 21s-30s & 24s-33s & 20s-27s & 23s-27s \\
         & $n=50$ & 27s-38s & 9s-23s & 25s-35s & 22s-35s \\
    \cline{1-6}
        \multirow[t]{2}{*}{TS-MR} & $n=7$ & 1.0s-1.3s & 0.9s-1.1s & 1.2s-2.2s & 1.0s-1.4s \\
        & $n=50$ & 1.4s-1.5s & 1.7s-2.0s & 2.3s-2.7s & 1.6s-2.0s \\
    \cline{1-6}
        \multirow[t]{2}{*}{HC} & $n=7$ & 17s-23s & 16s-24s & 19s-29s & 15s-24s \\
        & $n=50$ & 9s-12s & 8s-11s & 11s-14s & 12s-15s \\
    \cline{1-6}
        EGRA & $n=7$ & 0.34s-0.44s & 0.29s-0.45s & 0.29s-0.44s & 0.31s-0.50s \\
        & $n=50$ & 0.46s-0.58s & 0.98s-1.43s & 0.46s-0.66s & 0.55s-0.77s \\
    \bottomrule
    \end{tabular}
\end{table}

\begin{table}[htp]
    \centering
    \caption{Lower and upper quartiles of the time required to optimize the acquisition functions on the higher-dimensional test problems from the middle two columns of \cref{fig:results-rare} at several points during the optimization. The value of \(n\) denotes the number of the sample currently being optimized. On these test problems, the KG-MR algorithms were run using a GPU (NVIDIA L40 48 GB RAM) and 10 CPUs with a total of 59GB RAM. The remaining test problems use 8 CPUs and a total of 29.6GB RAM.}
    \label{tab:timings-highdim}
    \begin{tabular}{llrr}
    \toprule
     & & \thead{Hartmann (6D)} & \thead{Styblinski Tang\\(10D), cropped} \\
    \midrule
    \multirow[t]{3}{*}{Discrete KG-MR (GPU)} & $n=16$ & 1.6s-2.3s &  \\
     & $n=100$ & 2.2s-4.4s & 2.4s-4.3s \\
     & $n=200$ & 2.4s-7.1s & 2.8s-6.0s \\
    \cline{1-4}
    \multirow[t]{3}{*}{One-shot KG-MR (GPU)} & $n=16$ & 17s-26s &  \\
     & $n=100$ & 27s-39s & 25s-38s \\
     & $n=200$ & 28s-44s & 29s-37s \\
    \cline{1-4}
    \multirow[t]{3}{*}{TS-MR} & $n=16$ & 1.6s-3.3s &  \\
    & $n=100$ & 4.2s-5.9s & 1.9s-2.8s \\
    & $n=200$ & 4.6s-5.8s & 2.6s-3.5s \\
    \cline{1-4}
    \multirow[t]{3}{*}{HC} & $n=16$ & 1s-13s &  \\
    & $n=100$ & 2.2s-2.9s & 1s-11s \\
    & $n=200$ & 1.9s-3.4s & 1.6s-4.2s \\
    \cline{1-4}
    \multirow[t]{3}{*}{EGRA} & $n=16$ & 1.2s-3.6s &  \\
     & $n=100$ & 1.5s-2.0s & 1.7s-2.5s \\
     & $n=200$ & 1.5s-2.2s & 2.3s-4.3s \\
    \bottomrule
    \end{tabular}
\end{table}

\section{Problems with Less Extreme Optimal Failure Probabilities} \label[appendix]{sec:non-extreme-fail-probs}
To complement the problem from the main paper, we also study the algorithms on problems where the optimal failure probability is not as extreme. In this case, we do not need importance sampling and also drop the logarithm from the definition of KG-MR. That is, we use \(R_n(\vx) = P_n(\vx)\) instead of Equation~\eqref{eq:log-fail-prob}.

To achieve non-extreme optimal failure probabilities, the GP test problems were modified so that the proportion of \(\sYfeas\) which is a failure increases with dimension.
Additionally the standard deviation of the Gaussian perturbations was increased to 0.1.
The kernel and random seed used to generate the problems was not changed, so the actual sample \(f\) remained the same.
These changes are summarized in \cref{tab:gp-problem-parameters-nonextreme}.
For the non-GP test problems, the threshold was kept the same but the standard deviation of the noise variance was increased to the values shown in \cref{tab:nongp-problem-parameters-nonextreme}.

The results are shown in \cref{fig:results-nonextreme}.
The conclusions are similar to the extreme case in \cref{fig:results-rare}, with some exceptions. EGRA shows improved performance in the GP problems, while osKG-MR now matches the performance of EGRA on the (uncropped) Styblinski-Tang (10D) problem.
TS-MR fails to converge to the same optimum as the other algorithms on the Styblinski-Tang (2D) problem.
Generally, all algorithms converge faster, since the higher-variance perturbation distribution means that the failure probability landscape is simpler and less affected by small perturbations in the limit state surface.

\begin{table}[h]
    \centering
    \caption{Parameters used to generate the GP test problems with non-extreme failure probabilities.}
    \label{tab:gp-problem-parameters-nonextreme}
    \begin{tabular}{rrrl}
    \toprule
        Dimension & Threshold, \(c\) & Failure volume & Perturbation distribution, \(\mathbb{P}_\rvu\) \\
    \midrule
        2  & 3.6 & 33\% & \(\mathcal{N}\bigl(0, 0.1^2 \mI \bigr)\) \\
        8 & -1.4 & 67\% & \(\mathcal{N}\bigl(0, 0.1^2 \mI \bigr)\) \\
        10 & -4.5 & 90\% & \(\mathcal{N}\bigl(0, 0.1^2 \mI \bigr)\) \\
    \bottomrule
    \end{tabular}
\end{table}

\begin{table}[h]
    \centering
    \caption{Perturbation distributions used for the non-GP test problems with non-extreme failure probabilities}
    \label{tab:nongp-problem-parameters-nonextreme}
    \begin{tabular}{ll}
    \toprule
        Problem & Perturbation distribution \\
    \midrule
        Branin (2D) & \(\mathcal{N}\bigl(0, 2.5^2 \mI\bigr)\)\\
        Six hump camel (2D) & \(\mathcal{N}\bigl(0, \mathrm{diag}(0.6, 0.3)^2\bigr)\) \\
        Styblinski-Tang (2D) & \(\mathcal{N}\bigl(0, \mathrm{diag}(1, 2)^2\bigr)\) \\
        Ackley (2D) & \(\mathcal{N}\bigl(0, 8^2 \mI\bigr)\)\\
        Quadratic (2D) & \(\mathcal{N}\bigl(0, 0.12^2 \mI \bigr)\) \\
        Hartmann (6D) & \(\mathcal{N}\bigl(0, 0.1^2 \mI\bigr)\) \\
        Hartmann (6D), high threshold & \(\mathcal{N}\bigl(0, 0.18^2 \mI \bigr)\) \\
        Styblinski-Tang (10D) & \(\mathcal{N}\bigl(0, \mathrm{diag}(0.8, 0.8, 0.8, 0.2, \dots, 0.2)^2 \bigr)\) \\
        Styblinski-Tang (10D), cropped & \(\mathcal{N}\bigl(0, \mathrm{diag}(0.8, 0.8, 0.8, 0.2, \dots, 0.2)^2 \bigr)\) \\
    \bottomrule 
    \end{tabular}
\end{table}

\begin{figure}[htp]
    \centering
    \includegraphics[width=\linewidth]{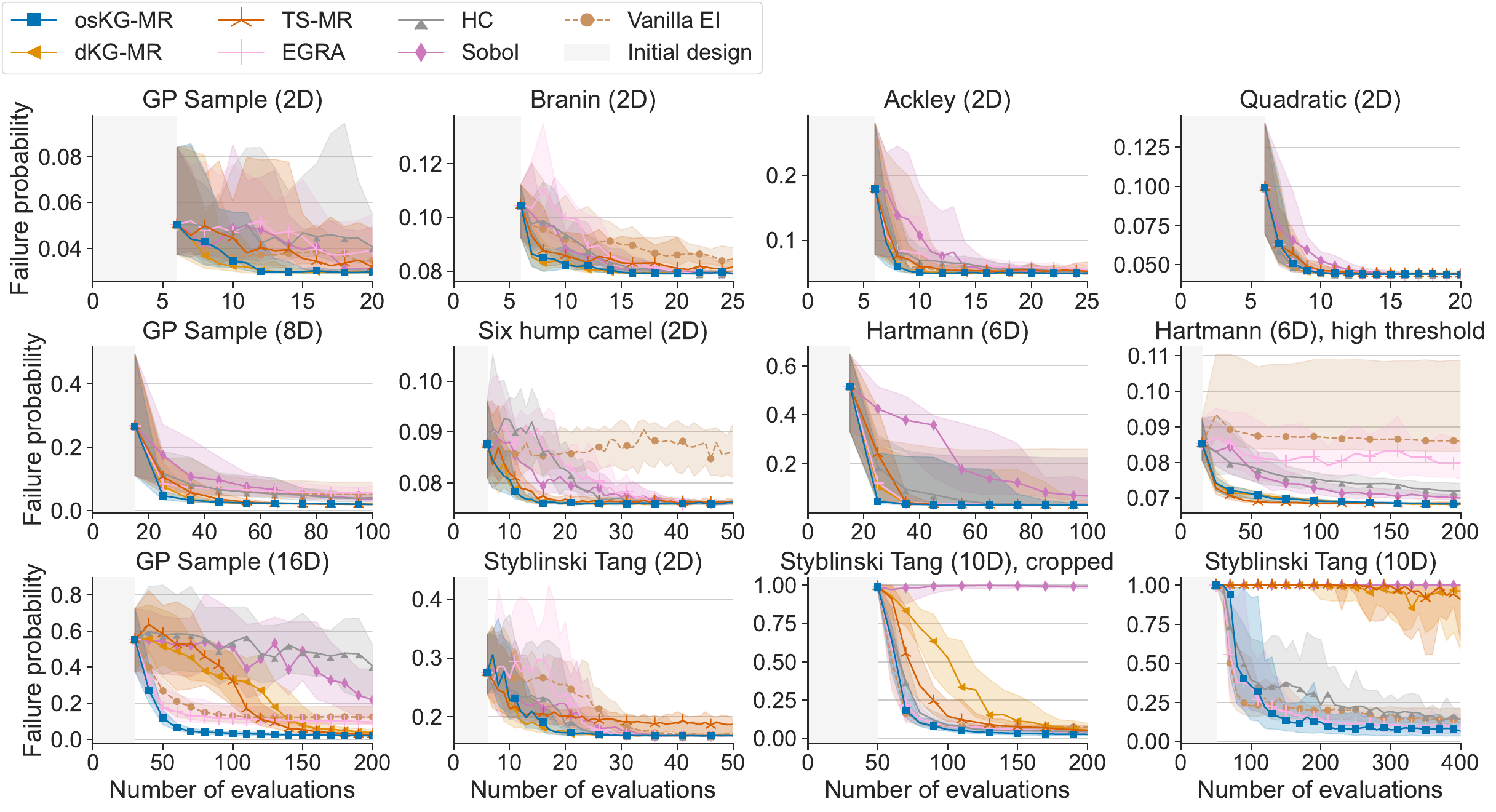}
    \caption[Results on the non-extreme reliability maximization test problems.]{Failure probabilities for the 12 test problems with a non-extreme optimal failure probability. As for \cref{fig:results-rare}, the first column contains the GP test problems, while the remaining columns are problems with potential for model mismatch. The last column contains problems for which KG-MR is not expected to have an advantage. The failure probability associated with the recommended solution is shown as a function of number of evaluations of the expensive black-box function. The solid lines show the median failure probability over 30 repeats and the shaded regions show the upper and lower quartiles.}
    \label{fig:results-nonextreme}
\end{figure}

\FloatBarrier
\section{Sensitivity Study} \label[appendix]{sec:sensitivity-study}
This section contains a study on the sensitivity of the one-shot and discrete KG-MR approximations to their parameters. The parameters studied are:
\begin{itemize}[nosep]
    \item the number \(N_u\) of qMC points used to estimate the expectation \(\E_\rvu\),
    \item the number \(N_v\) of qMC points used to estimate the expectation over the next potential sample \(v_{n+1}\),
    \item the number \(N_x = |\sXd|\) of points in the discretization of \(\sX\) (dKG-MR only),
    \item the number \(N_\text{restarts}\) of restarts used when optimizing the acquisition function with L-BFGS-B.
\end{itemize}
Sensitivity is tested on the problems in the middle two columns of \cref{fig:results-rare} and results are presented in \cref{fig:sensitivity-oskg,fig:sensitivity-dkg}. Both algorithms are reasonably robust to all parameters. However,
\begin{itemize}[nosep]
    \item dKG-MR is sensitive to the size of \(N_x\) in the cropped 10D Styblinski-Tang problem,
    \item a higher \(N_u\) has a small benefit to osKG-MR in the cropped 10D Styblinski-Tang,
    \item a lower \(N_v\) is slightly harmful to osKG-MR in the Hartmann (6D) problem.
\end{itemize}

\begin{figure}[htbp]
    \centering
    \includegraphics[width=\linewidth]{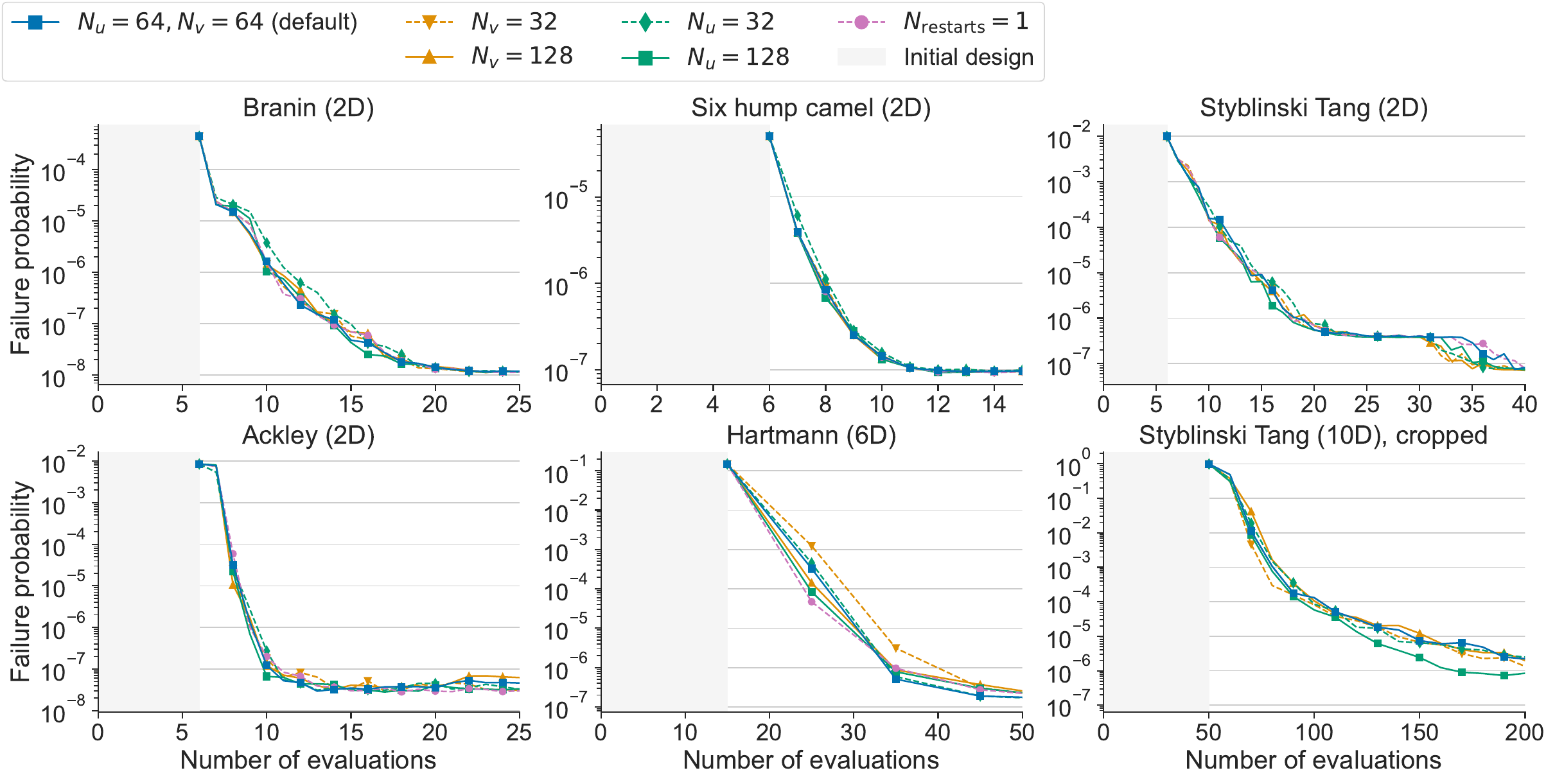}
    \caption{Sensitivity analysis on the one-shot KG-MR approximation (osKG-MR) on the problems with extremely small optimal failure probabilities in the middle two columns of \cref{fig:results-rare}. The failure probability associated with the recommended solution is shown as a function of the number of evaluations of the expensive black-box function. The solid lines show the median failure probability over 30 repeats. The upper and lower quartiles are omitted for clarity. Only parameters which differ from the default are shown in the legend.}
    \label{fig:sensitivity-oskg}
\end{figure}

\begin{figure}[htbp]
    \centering
    \includegraphics[width=\linewidth]{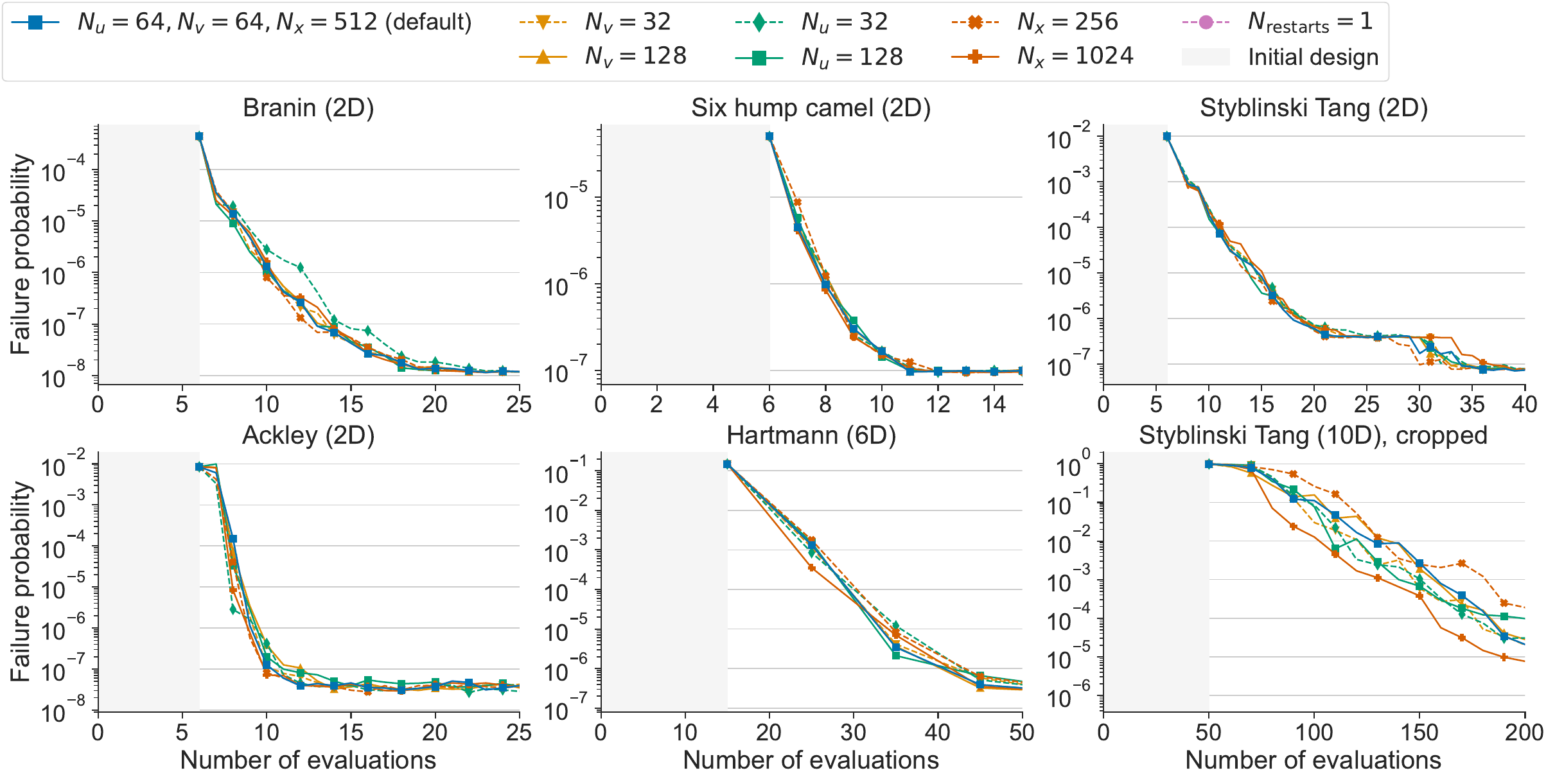}
    \caption{Sensitivity analysis on the discrete KG-MR approximation (dKG-MR) on the problems with extremely small optimal failure probabilities in the middle two columns of \cref{fig:results-rare}. The failure probability associated with the recommended solution is shown as a function of the number of evaluations of the expensive black-box function. The solid lines show the median failure probability over 30 repeats. The upper and lower quartiles are omitted for clarity. Only parameters which differ from the default are shown in the legend.}
    \label{fig:sensitivity-dkg}
\end{figure}

%%%%%%%%%%%%%%%%%%%%%%%%%%%%%%%%%%%%%%%%%%%%%%%%%%%%%%%%%%%%%%%%%%%%%%%%%%%%%%%
%%%%%%%%%%%%%%%%%%%%%%%%%%%%%%%%%%%%%%%%%%%%%%%%%%%%%%%%%%%%%%%%%%%%%%%%%%%%%%%

\end{document}